\documentclass{article}
\usepackage{log_2023}            

\usepackage{booktabs}            
\usepackage{multirow}            
\usepackage{amsfonts}            
\usepackage{graphicx}            
\usepackage{duckuments}          

\usepackage[numbers,compress,sort]{natbib}

\usepackage[ruled,vlined,linesnumbered]{algorithm2e}
\usepackage{algpseudocode}
\usepackage{balance}
\usepackage{subcaption}
\usepackage{bbold}
\usepackage{enumitem}
\usepackage{xcolor, colortbl}
\usepackage{amsmath, amsthm}
\usepackage{url}
\usepackage{multicol}
\usepackage{graphicx}
\usepackage{subcaption}
\usepackage{float}
\usepackage{pifont}
\usepackage{wrapfig}
\usepackage[toc,page]{appendix}

\usepackage{amsmath,amsfonts,bm}

\def\1{\bm{1}}

\DeclareMathAlphabet{\mathsfit}{\encodingdefault}{\sfdefault}{m}{sl}
\SetMathAlphabet{\mathsfit}{bold}{\encodingdefault}{\sfdefault}{bx}{n}

\newcommand{\ie}{\textit{i}.\textit{e}.}

\SetCommentSty{mycommentfont}

\newcommand{\hide}[1]{}

\newcommand{\spmetric}{\Delta_{\text{SP}}}
\newcommand{\eometric}{\Delta_{\text{EO}}}
\newcommand{\method}{BeMap}

\newtheorem{problem}{Problem}
\newtheorem{lemma}{Lemma}
\newtheorem{theorem}{Theorem}
\newtheorem{definition}{Definition}
\newtheorem{proposition}{Proposition}
\newtheorem{assumption}{Assumption}

\title{\method: Balanced Message Passing for Fair Graph Neural Network} 

\author[Lin et al.]{Xiao Lin$^1$,~~~Jian Kang$^2$,~~~Weilin Cong$^3$,~~~Hanghang Tong$^1$ \\
    $^1$University of Illinois Urbana-Champaign, $^2$University of Rochester, $^3$Penn State University \\
    \email{$^1$\{xiaol13, htong\}@illinois.edu, $^2$jian.kang@rochester.edu, $^3$wxc272@psu.edu}
}

\begin{document}
\maketitle

\vspace{-6mm}
\begin{abstract}
Fairness in graph neural networks has been actively studied recently. However, existing works often do not explicitly consider the role of message passing in introducing or amplifying the bias.
In this paper, we first investigate the problem of bias amplification in message passing. We empirically and theoretically demonstrate that message passing could amplify the bias when the 1-hop neighbors from different demographic groups are unbalanced. Guided by such analyses, we propose \method, a fair message passing method, that leverages a balance-aware sampling strategy to balance the number of the 1-hop neighbors of each node among different demographic groups. Extensive experiments on node classification demonstrate the efficacy of \method~in mitigating bias while maintaining classification accuracy. The code is available at \url{https://github.com/xiaolin-cs/BeMap}.
\end{abstract}
\vspace{-6mm}
\section{Introduction}
\vspace{-1mm}

Graph neural network (GNN) learns node representations via message passing, which aggregates information from local neighborhood of each node, to capture topological and attributive characteristics of graph data~\cite{kipf2016semi, hamilton2017inductive, chen2018fastgcn, wu2019simplifying}. Despite the substantial research progress, concrete evidence has shown that GNNs could carry certain biases, which lead to unfair learning results. For example, a graph-based recommender system may discriminate individuals from certain groups by recommending fewer career opportunities~\cite{sweeney2013discrimination, lambrecht2019algorithmic}. As such, the widespread use of GNNs in safe-critical and life-changing applications, such as employment systems~\cite{mehrabi2021survey}, is in imminent need of fairness considerations.

In this paper, we study group fairness, which is one of the most intuitive and fundamental fairness notions. To date, researchers have developed a collection of algorithms to ensure group fairness for graph neural networks~\cite{bose2019compositional, dai2021say, li2021dyadic, dong2022edits, jiang2022fmp}. A typical approach is to impose the fairness consideration as a regularization term from the optimization perspective. Though achieving promising performance in bias mitigation, these methods often do not explicitly take into account the bias introduced or amplified by message passing during model training. To our knowledge, only a few works consider bias mitigation during message passing~\cite{spinelli2021biased, jiang2022fmp}. They either manipulate the graph structure empirically~\cite{spinelli2021biased, chen2022graph} or completely change the message passing procedure to solve a complex minimax problem~\cite{jiang2022fmp, zhu2023fairness}. Several fundamental questions remain nascent in the context of algorithmic fairness on GNN: \textit{(Q1) Does the message passing in GNN introduce or amplify bias? (Q2) If so, how to ensure fairness during message passing without changing the GNN architecture in a principled way?}

To answer these questions, we study the problems of (Q1) bias amplification in message passing and (Q2) fair message passing. For (Q1) bias amplification in message passing, we both empirically and theoretically show that message passing can indeed amplify bias when the numbers of neighboring nodes from different demographic groups are unbalanced. The key idea is to measure the bias as the reciprocal of the expected squared Euclidean distance of each node to the centroid of the demographic group it belongs to.\footnote{The more biased the model, the smaller the distance.} Based on the distance-based bias definition, for each node, if the numbers of neighbors from different demographic groups are unbalanced, the expected squared distance to the corresponding centroid would shrink after message passing, meaning that the dependence of the model on the sensitive attributes increases, and the bias in predictions is thus amplified. For (Q2) fair message passing in GCNs, guided by the theoretical insights from Q1, we propose \method~that leverages a balance-aware sampling strategy to selectively sample a subset of neighbors to achieve a balanced number of neighbors from different demographic groups for bias mitigation.

In summary, our main contributions are as follows:
\vspace{-1mm}
\begin{itemize}[
    align=left,
    leftmargin=2em,
    itemindent=0pt,
    labelsep=0pt,
    labelwidth=1em,
]
    \vspace{-1mm}
    \item \textbf{Problems.} We study the problems of bias amplifications in message passing and fair message passing from the perspective of neighborhood balance of each node. 
    \vspace{-1mm}
    \item \textbf{Analyses.} Our analyses \textit{both empirically and theoretically} reveal that the message passing schema could amplify the bias, if the demographic groups with respect to a sensitive attribute in the local neighborhood of each node are unbalanced.
    \vspace{-1mm}
    \item \textbf{Algorithm.} Guided by our analyses, we propose an easy-to-implement sampling-based method named \method, which leverages a balance-aware sampling strategy to mitigate bias.
    \vspace{-1mm}
    \item \textbf{Evaluations.} We conduct extensive experiments on real-world datasets, which demonstrate that \method~could significantly mitigate bias while maintaining a competitive accuracy compared with the baseline methods.
\end{itemize}

\vspace{-4mm}
\section{Related Work}
\vspace{-2mm}
\noindent\textbf{Graph neural network (GNN)} has demonstrated state-of-the-art performance in various learning tasks, including anomaly detection~\cite{ding2021few}, crime rate prediction~\cite{suresh2019framework}, and recommender systems~\cite{fan2019graph}.
\cite{kipf2016semi} proposes the Graph Convolutional Network (GCN) by utilizing a localized first-order approximation of spectral graph convolutions. 
\cite{hamilton2017inductive} introduces graph neural networks for inductive learning via neighborhood sampling and aggregation. 
\cite{velivckovic2017graph} leverages the self-attention mechanism to learn the attention weights for neighborhood aggregation. 
\cite{chen2018fastgcn} scales up the training on large graphs by importance sampling. 
\cite{feng2020graph} learns node representations by randomly dropping nodes to augment data and enforcing the consistency of predictions among augmented data. 
Similarly, \cite{papp2021dropgnn} randomly drops nodes for several runs and aggregates the predictions for the final prediction by ensemble. Different from~\cite{chen2018fastgcn, feng2020graph, papp2021dropgnn} that drop nodes for scalable training or improving the generalization and/or expressiveness of GNN, our work drops edges (i.e., dropping nodes in the local neighborhood of a node) to mitigate bias in GNN. 
\cite{rong2019dropedge} randomly drops edges to perform data augmentation and prevents over-smoothing. 
Different from~\cite{rong2019dropedge} that randomly drops edges to alleviate over-smoothing, our work selectively removes edges to obtain balanced graph structures to improve the model fairness. 
For comprehensive literature reviews on graph neural networks, please refer to \cite{zhang2019graph, wu2022recent, han2022geometrically, xia2022survey, ding2022data}.

\vspace{-1mm}
\noindent \textbf{Fairness in graph neural networks} has been actively studied. 
A common strategy to learn fair graph neural networks is through optimizing a regularization-based optimization problem~\cite{bose2019compositional, agarwal2021towards, dai2021say, dong2022edits, ma2022learning}. 
For group fairness, 
\cite{bose2019compositional} ensures group fairness by minimizing the mutual information between sensitive attributes and node embeddings via adversarial learning.
\cite{dai2021say} leverages a similar debiasing strategy to learn fair graph neural networks with limited sensitive attributes. 
For individual fairness, 
\cite{kang2020inform} mitigates individual bias using the Laplacian regularization on the learning results.
\cite{dong2021individual} adopts learning-to-rank for ranking-based individual fairness.
In terms of counterfactual fairness, 
\cite{agarwal2021towards} learns counterfactually fair node embeddings through contrastive learning.
\cite{ma2022learning} generates the counterfactual graph with variational graph auto-encoder~\cite{kipf2016variational}.
Different from the aforementioned techniques, our work does not require any regularization in the objective function.
Other than regularization-based formulation, 
\cite{dong2022edits} preprocesses the input graph by minimizing the Wasserstein distance between the embedding distributions of majority and minority groups. 
\cite{chen2022graph} generates fair neighborhoods by neighborhood rewiring and selection.
\cite{jiang2022fmp} proposes a novel message passing schema that solves a minimax problem w.r.t. group fairness. 
\cite{spinelli2021biased} promotes the existence of edges connecting nodes in different demographic groups. 
However, 
\cite{jiang2022fmp} bears little resemblance to the original message passing schema, thus completely changing the training procedures; 
\cite{chen2022graph} lacks a theoretical understanding on the connection between balanced neighborhood and group fairness; 
\cite{dong2022edits} requires a pre-trained model to pre-process the input graph; 
and \cite{spinelli2021biased} cannot guarantee a balanced neighborhood consequently. 
Compared to \cite{jiang2022fmp, dong2022edits, chen2022graph, spinelli2021biased}, our work fills the gap between a balanced neighborhood and group fairness and tunes the message passing in a easy-to-implement node sampling strategy with theoretical guarantee. Further discussion on prior works related to fairness in GNNs is provided in Appendix \ref{appdix:related_work}. 

\vspace{-5mm}
\section{Preliminaries} \label{sec:prelim}
\vspace{-1mm}
In this section, we first briefly introduce the Graph Convolutional Network (GCN) and the commonly used group fairness definitions. Then, we formally define the problem of bias amplification in message passing and fair message passing in GCN. 

\vspace{-1mm}
\noindent\textbf{Notation Convention.} We use bold uppercase/lowercase letters for matrices/vectors (e.g., $\mathbf{A}$, $\mathbf{x}$), 
italic letters for scalars (e.g., $d$), and calligraphic letters for sets (e.g., $\mathcal{N}$). For matrix indexing, the $i$-th row of a matrix is denoted as its corresponding bold lowercase letter with subscript $i$ (e.g., the $i$-th row of matrix $\mathbf{X}$ is $\mathbf{x}_i$). Notations are summarized in Table \ref{tab:notation_appendix}.

\vspace{-1mm}
\noindent\textbf{Graph Convolutional Networks.} In this paper, we study the message passing schema in Graph Convolutional Network (GCN)~\cite{kipf2016semi}, which is one of the most classic graph neural networks. Let $\mathcal{G} = \left\{\mathcal{V}, \mathbf{A}, \mathbf{X}\right\}$ denote a graph with a node set of $n$ nodes $\mathcal{V} = \left\{v_1, \dots, v_n \right\}$, a binary adjacency matrix $\mathbf{A}$, and node feature matrix $\mathbf{X}$. For any node $v_i$, we denote its degree and the feature as $d_i$ and $\mathbf{x}_i$. \footnote{Although we only consider binary adjacency matrix, our theoretical analysis and proposed method can be naturally generalized to graph with weighted edges by replacing 0/1 edge weight to other values.} 
For the $l$-th hidden layer in an $L$-layer GCN, we denote its weight matrix as $\mathbf{W}^{\left(l\right)}$ and the input and output representations of node $v_i$ as $\mathbf{h}_i^{\left(l\right)}$ and $\mathbf{h}_i^{\left(l+1\right)}$, respectively.\footnote{For notation simplicity, we denote $\mathbf{h}_i^{\left(1\right)} = \mathbf{x}_i$ for all $v_i\in\mathcal{V}$.} 
Then the message passing schema in GCN is $\mathbf{\widehat h}^{\left(l\right)}_i = \sum_{v_j \in \mathcal{\widehat N}\left(v_i\right)} \alpha_{ij} \mathbf{h}_j^{\left(l\right)}$, where $\mathcal{\widehat N}\left(v_i\right) = \mathcal{N}\left(v_i\right)\cup \left\{v_i\right\}$ is the self-augmented neighborhood (i.e., the union of node $v_i$ and its 1-hop neighborhood $\mathcal{N}\left(v_i\right)$) and $\alpha_{ij}$ is the aggregation weight with the source node being $v_i$ and the target node being $v_j$ (e.g., $\alpha_{ij} = \frac{1}{\sqrt{d_i + 1} \sqrt{d_j + 1}}$ for symmetric normalization
and $\alpha_{ij} = \frac{1}{d_i + 1}$ for row normalization). 
Based on the message passing schema, the graph convolution for $v_i$ in GCN can be formulated as $\mathbf{h}_i^{(l+1)} = \sigma\left(\mathbf{\widehat h}^{\left(l\right)}_i \mathbf{W}^{\left(l\right)}\right)$, where $\sigma(\cdot)$ is the nonlinear activation (e.g., ReLU).

\vspace{-2mm}
\noindent\textbf{Group Fairness.} Group fairness aims to ensure the parity of model predictions among the demographic groups of data points, where the demographic groups are often determined by a sensitive attribute (e.g., gender and race). Specifically, we adopt two widely used fairness criteria, i.e., statistical parity~\cite{dwork2012fairness} and equal opportunity~\cite{hardt2016equality}, which are defined in Definitions~\ref{defn:sp} and \ref{defn:eo}, respectively.

\begin{definition}\label{defn:sp}
\textbf{(Statistical parity~\cite{dwork2012fairness})} Given any label $y \in \left\{0, 1\right\}$, any sensitive attribute $s\in\left\{0, 1\right\}$ and the prediction $\widehat{y} \in \left\{0, 1\right\}$ inferred by a model, a model satisfies statistical parity if and only if the acceptance rate with respect to the model predictions are equal for different demographic groups. Mathematically, statistical parity can be expressed as
\vspace{-1mm}
\begin{equation}\label{eq:sp}
    \textit{Pr}\left(\widehat{y} = 1 \mid s = 0\right) = \textit{Pr}\left(\widehat{y} = 1 \mid s = 1\right)
\end{equation}
where $\textit{Pr}\left(\cdot\right)$ refers to the probability of an event.
\end{definition}

\begin{definition}\label{defn:eo}
\textbf{(Equal opportunity~\cite{hardt2016equality})} Following the settings of Definition~\ref{defn:sp}, a model satisfies equal opportunity if and only if the true positive rate with respect to the model predictions are equal for different demographic groups. Mathematically, equal opportunity can be expressed as
\vspace{-1mm}
\begin{equation}\label{eq:eo}
    \textit{Pr}\left(\widehat{y} = 1 \mid y = 1, s = 0\right) = \textit{Pr}\left(\widehat{y} = 1 \mid y = 1, s = 1\right)
\end{equation}
where $\textit{Pr}\left(\cdot\right)$ refers to the probability of an event.
\end{definition}

Given Definitions \ref{defn:sp} and \ref{defn:eo}, the bias with respect to statistical parity and equal opportunity are naturally defined as the discrepancies in the acceptance rate and the true positive rate across different demographic groups. Mathematically, the quantitative measures of bias with respect to statistical parity and equal opportunity are defined as Eq.~\eqref{sp_metric} and Eq. \eqref{eo_metric}, respectively. 
\vspace{-1mm}
\begin{equation}\label{sp_metric}
    \spmetric = \vert \textit{Pr}\left(\widehat{y} = 1 \mid s = 1\right) - \textit{Pr}\left(\widehat{y} = 1 \mid s = 0\right) \vert
\end{equation}
\begin{equation}\label{eo_metric}
   \eometric = \vert \textit{Pr}\left(\widehat{y} = 1 \mid y = 1, s = 1\right)-\textit{Pr}\left(\widehat{y} = 1 \mid y = 1,s = 0\right) \vert
\end{equation}
\vspace{-2mm}
From the information-theoretic perspective, minimizing Eq.~\eqref{sp_metric} and Eq. \eqref{eo_metric} is essentially equivalent to eliminating the statistical dependency between the model prediction $\widehat{y}$ and the sensitive attribute $s$. 
Consequently, to ensure group fairness, existing works~\cite{bose2019compositional, dai2021say, dong2022edits} propose to impose the statistical dependency (e.g., mutual information) between $\widehat{y}$ and $s$ as regularization in the optimization problem.
Nevertheless, these works could not explicitly consider the bias caused by the message passing schema. To this end, we seek to understand the role of message passing in the context of algorithmic fairness. 
Based on that, we define the problem of bias amplification in message passing as follows:

\begin{problem}
Bias amplification in message passing.
\end{problem}

\vspace{-2mm}
\noindent \textbf{Given}: (1) An undirected graph $\mathcal{G}=\left\{\mathcal{V}, \mathbf{A}\right\}$; (2) an $L$-layer GCN; (3) the vanilla message passing schema in any $l$-th hidden layer $\mathbf{\widehat h}^{\left(l\right)}_i = \sum_{v_j \in \mathcal{N}\left(v_i\right) \cup \left\{v_i\right\} } \alpha_{ij} \mathbf{h}_j^{\left(l\right)}$; (4) a sensitive attribute $s$; (5) a bias measure $\textit{bias}$ for statistical parity.

\vspace{-1mm}
\noindent \textbf{Find}: A binary decision regarding whether or not the bias will be amplified after message passing.



Based on the answer to Problem 1, our goal is to develop a generic fair message passing such that the bias measure in Problem 1 will decrease after message passing. We define the problem of fair message passing as follows:


\begin{problem}
Fair message passing.
\end{problem}

\vspace{-2mm}
\noindent \textbf{Given}: (1) An undirected graph $\mathcal{G}=\left\{\mathcal{V}, \mathbf{A}, \mathbf{X}\right\}$; (2) an $L$-layer GCN; (3) a sensitive attribute $s$; (4) a bias measure $\spmetric$.

\vspace{-2mm}
\noindent \textbf{Find}: A fair message passing $\mathbf{\widehat h}^{\left(l\right)}_i = \textit{MP}\left(\mathbf{h}_i^{\left(l\right)}, \mathcal{G} \right)$ such that $\spmetric$ decreases after message passing.

\vspace{-5mm}
\section{Bias Amplification in Message Passing}\label{sec:bias_amplification}
\vspace{-1mm}
In this section, we provide both the empirical evidence and the theoretical analysis on bias amplification in message passing.

\begin{figure*}[tbp]
    \centering
    \includegraphics[width=\linewidth]{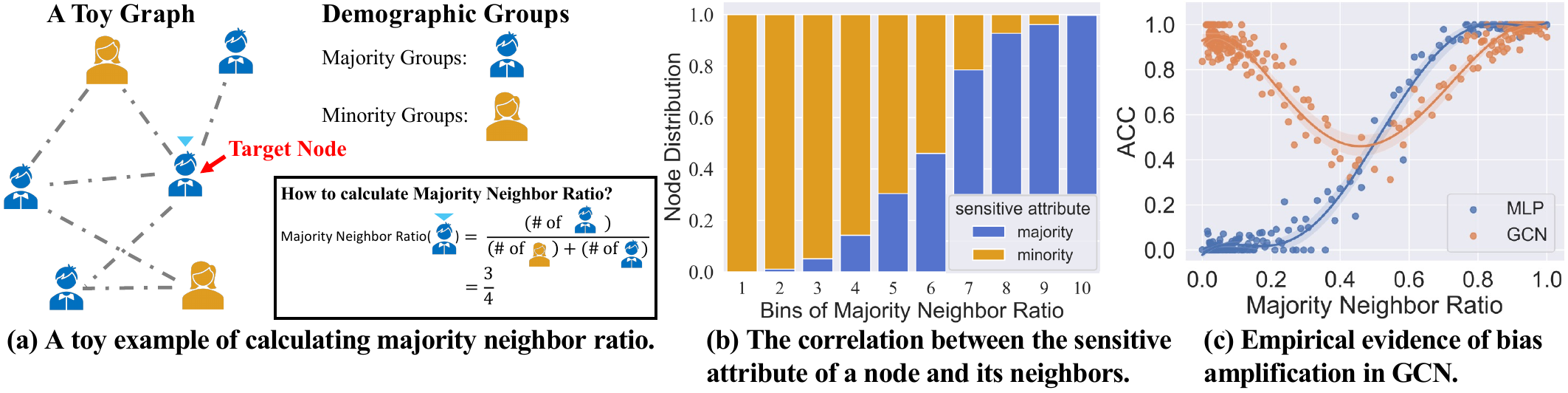}
    \caption{The empirical evidence of bias amplification in GCN on the Pokec-z dataset. Best viewed in color. In (b), majority neighbor ratio is grouped into 10 equal-width bins with width being $0.1$, i.e., $[0, 0.1), [0.1, 0.2),\ldots,[0.9, 1.0]$.}
    \vspace{-6mm}
    \label{fig:empirical_evidence}
\end{figure*}

\vspace{-3mm}
\subsection{Empirical Evidence}\label{subsec:empirical_evidence}
\vspace{-2mm}
We first empirically illustrate that message passing could amplify the bias. 
The design of our experiment is based on the assumption that the learning results (e.g., class probabilities) from a fair model (e.g., FairGNN~\cite{dai2021say}) should be independent of the sensitive attribute~\cite{dai2021say}. Therefore, we use a logistic regression classifier to predict the sensitive attribute of a node using its corresponding learning result output by a model, where the accuracy of the sensitive attribute estimator naturally serves as an indicator of how biased a model is. 

We investigate the effect of bias amplification of message passing on the Pokec-z dataset by comparing the behavior of a two-layer multi-layer perception (MLP) and a two-layer GCN.\footnote{Details about the Pokec-z dataset is deferred to Section~\ref{sec:experiments} and Appendix~\ref{sec:appendix_datasets}.} Note that the only difference between the MLP and GCN is whether the message passing is utilized or not, while the hidden dimensions and the nonlinear activation function are the same with each other. The details for the empirical evaluation are as follows. First, we train the MLP and GCN to predict labels using the node features without sensitive attributes as input. When the graph neural networks reach the best performance, we freeze the parameters of the first layer of both the MLP and GCN. Hence, the first layer of the MLP and GCN can serve as embedding extractors. Second, we train two logistic regression classifiers to predict sensitive attributes with the embeddings extracted from the MLP and GCN, respectively. If the classifier can accurately predict the sensitive attribute, it implies that those embeddings contain rich sensitive information. In this way, we use the accuracy as a proxy to measure the dependency between sensitive attributes and embeddings extracted from the two models.

\vspace{-1mm}
The evaluation results are presented in Figure~\ref{fig:empirical_evidence}. From the figure, we have three key observations. First, from Figure~\ref{fig:empirical_evidence}(b), we can observe a strong positive correlation between the sensitive attribute of a node and the sensitive attribute of its neighbors, i.e., a node tends to have the same sensitive attribute as the majority of its neighbors. Second, as shown in Figure~\ref{fig:empirical_evidence}(c), the MLP predicts the sensitive attribute of all nodes as the sensitive attribute of the majority demographic group, even when a node and all its neighbors do not belong to the majority demographic group. Third, in Figure~\ref{fig:empirical_evidence}(c), the GCN makes much more accurate predictions than the MLP when the neighbors of a node are more likely to be in the minority group. 
The first and second observations imply that the embeddings extracted from the MLP have little correlation with sensitive attributes, thus leading to the less biased predictions. On the contrary, the GCN extracts the embeddings containing rich sensitive information for those nodes that belong to the same demographic group as their neighbors, as indicated in the third observation, regardless of the minority group and the majority group.
These results demonstrate that message passing could aggregate the sensitive information from the neighborhood, which further amplifies the bias and guides the logistic regression classifier to correctly predict the sensitive attribute.

\vspace{-2mm}
\subsection{Theoretical Analysis}\label{subsec:theoretical_analysis}
\vspace{-2mm}
\begin{wrapfigure}{r}{0.62\textwidth} 
\vspace{-4mm}
\begin{center}
    \includegraphics[width=0.6\textwidth]{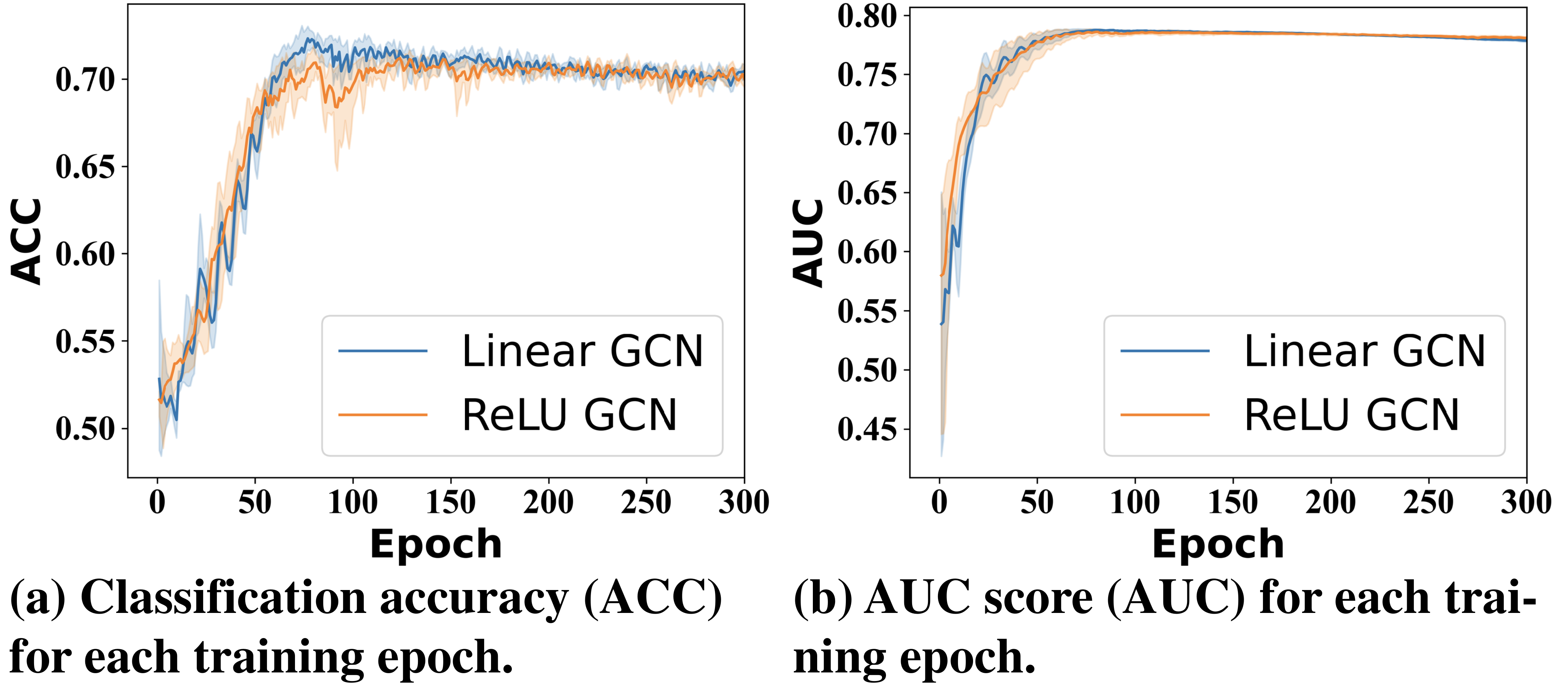}
    \caption{Node classification accuracy (ACC) and AUC score (AUC) curves of linear GCN (Linear GCN) vs. nonlinear GCN (ReLU GCN) on the NBA dataset.}
    \label{fig:expressive_power}
\end{center}
\vspace{-6mm}
\end{wrapfigure} 

For analysis purposes, we consider a binary sensitive attribute and an $L$-layer linear GCN with row normalization (i.e., $\alpha_{ij} = \frac{1}{d_i + 1}$) for binary node classification, where linear GCN is a special type of GCN model without the nonlinear activation (e.g.,~\cite{wu2019simplifying}).\footnote{Our analysis can be generalized to non-binary sensitive attribute and multi-class classification.} Although our analysis relies on linearity, recent works have shown that lack of nonlinearity in GCN enjoys almost the same expressive power as nonlinear GCN~\cite{xu2021optimization, wei2022understanding}. Moreover, as shown in Figure~\ref{fig:expressive_power}, linear GCN exhibits similar or even slightly better node classification accuracy than its nonlinear counterpart, which also demonstrates that linearizing GCN is a good alternative to understand the behavior of GNN models.


\vspace{-1pt}
To analyze the connection between the bias amplification phenomenon and graph topology, we make the assumption (Assumption~\ref{asmp:random_graph}) concerning the generation of graphs. Specifically, we use random graphs to analyze the key properties in Proposition \ref{prop:graph_property} of GNNs, which has been widely adopted in previous works~\cite{keriven2020convergence, jiang2022fmp, li2023ood}. 
For example, FMP~\cite{jiang2022fmp} investigates the connection between group fairness and graph properties using random graphs as the basis for analysis. OOD-GNN~\cite{li2023ood} examines the out-of-distribution generalization abilities of GNNs on specific random graphs with distribution shifts. And Keriven et al.~\cite{keriven2020convergence} explores the convergence and stability of GCNs by analyzing their behavior on standard models of random graphs.

\vspace{-1pt}
\begin{assumption} \label{asmp:random_graph}
    The graphs discussed in theoretical analysis are generated by the Gilbert random graph model \cite{zhou2009graph}, in which each edge is added independently with a fixed possibility $P$.
\end{assumption}

\vspace{-1mm}
To measure how biased a GNN is after message passing, we further make the following assumption on the existence of ideal fair node embedding. The intuition is that a fair graph neural network would output fair node embeddings without sensitive information after the last hidden layer. 
Following this intuition, a node embedding output by any GNN can be decomposed into a fair node embedding and the residual (i.e., for a fair GNN, the residual would be $0$).

\begin{assumption} \label{asmp:exist_fair_node_emb}
    The output node embedding $\mathbf{z}$ from the last hidden layer is a linear combination of fair embedding $\mathbf{z}_t$ and bias residual $\mathbf{z}_b$, i.e., $\mathbf{z} = \mathbf{z}_t + \mathbf{z}_b$, where the fair embedding gives fair prediction while the bias residual corresponds to the bias in predictions. Written in matrix form, we have $\mathbf{Z} = \mathbf{Z}_t + \mathbf{Z}_b$.
\end{assumption}

\vspace{-1mm}
Supposing that Assumption \ref{asmp:exist_fair_node_emb} holds, we show in Lemma \ref{lm:linear_node_embed} that for any $v_i \in \mathcal{V}$, the output embedding from any hidden layer can be viewed as a linear combination of fair embedding and bias residual. Proof is deferred in Appendix \ref{appendix:linear_combine}.

\begin{lemma} \label{lm:linear_node_embed}
    (\textbf{Linear decomposition of a node embedding}) Suppose that Assumption \ref{asmp:exist_fair_node_emb} holds. Given an input graph $\mathcal{G} = \left\{\mathcal{V}, \mathbf{A}, \mathbf{X}\right\}$ and an $L$-layer linear GCN with row normalization, we have
    \vspace{-2mm}
    \begin{equation}
        \mathbf{h}_i^{\left(l\right)} = \left(\mathbf{t}_i^{\left(l\right)} + \mathbf{b}_i^{\left(l\right)}\right) \mathbf{W}^{\left(1\right)} \dots \mathbf{W}^{\left(l-1\right)}
    \end{equation}
    \vspace{-5mm}
    
    \noindent for any $v_i \in \mathcal{V}$ and any hidden layer $l\in\left\{1, \dots, L\right\}$, where $\mathbf{t}_i^{\left(l\right)} = \widetilde{\mathbf{A}}^{l-1}(\widetilde{\mathbf{A}}^L)^{\dagger}  \mathbf{z}_t \mathbf{W}^{\dagger}$ is the fair embedding,  $\mathbf{b}_i^{\left(l\right)} = \widetilde{\mathbf{A}}^{l-1}(\widetilde{\mathbf{A}}^L)^{\dagger} \mathbf{z}_b \mathbf{W}^{\dagger}$ is the bias residual, $\widetilde{\mathbf{A}} = \mathbf{D}^{-1} \mathbf{A}$ is the row-normalized adjacency matrix, and $\mathbf{W}^{\dagger}$ is the Moore–Penrose inverse of $\mathbf{W} = \mathbf{W}^{\left(1\right)}\mathbf{W}^{\left(2\right)}\dots \mathbf{W}^{\left(L\right)}$.
\end{lemma}

\vspace{-1mm}
Lemma \ref{lm:linear_node_embed} offers a fresh perspective on the interplay between message passing and fairness. In particular, we will show that message passing alters the distribution of bias residuals associated with all nodes, thereby influencing fairness. Specifically, when the bias residuals of all nodes converge to the same point after message passing, the output node embeddings from the last hidden layer equals to the sum of fair node embeddings and a constant vector, resulting in fair prediction results. Conversely, if the differences between the bias residuals from different demographic groups are magnified through message passing, the unfairness in prediction results will be also amplified. Therefore, the distributional shift in the bias residual reveals how message passing would affect the fairness.

For GCNs with row normalization, message passing essentially performs a weighted average over the bias residuals from local neighborhoods. Since the summation of the weight from all the neighbors is $1$ for row normalization, the centroid of the distribution of the bias residual for any demographic group remains unchanged during message passing. Please see the detailed proof in Appendix \ref{appendix:GCN_analyse}. In this case, an increase in the difference of bias residuals is directly related to the shrinkage in the expected distance between the bias residual of any node and the centroid of its corresponding demographic group. For example, if the expected distance is much larger than the distance between centroids, then the distributions of bias residuals from various demographic groups likely exhibit substantial overlapping areas. Within this overlap, identifying the demographic group to which a specific bias residual belongs becomes challenging. This suggests that when the distance between the centroids of two demographic groups is fixed, an increase in the expected distance results in a considerable expansion of the overlapping area. Consequently, the difference between the two distributions diminishes. Meanwhile, an extreme unfair situation is that the bias residual of any node exactly falls into the centroid of its demographic group. Then the bias residual essentially becomes an indicator for sensitive attributes.
As a result, even if the input data excludes any information related to sensitive attributes, GNNs can readily distinguish the sensitive attribute associated with each node. It grants GNNs a convenient path to infer labels through the leaked information of sensitive attributes, thus leading to severe unfairness.

To better understand the change of the expected distances during message passing, we introduce a distance-based bias that captures the model bias. Let us denote the centroid of the bias residuals from the majority as $\mu_0$ and denote the centroid from the minority as $\mu_1$. 

\begin{definition}\label{defn:distance_based_bias}
(\textbf{Distance-based bias}) Suppose that we have an input graph $\mathcal{G} = \{\mathcal{V}, \mathbf{A}, \mathbf{X}\}$ and an $L$-layer GCN. For any $l$-th hidden layer, the distance-based bias is defined as the reciprocal of the expected squared Euclidean distance from the bias residual of each node to the centroid of its corresponding demographic group, which is shown below.
\vspace{-1mm}
\begin{equation}\label{eq:distance_based_bias}
    \text{bias}^{\left(l\right)}(\mathcal{G}, \mathbf{B}^{\left(l\right)}, s) = \frac{1}{\mathbb{E}_{v_i}[\|\mathbf{b}^{\left(l\right)}_i - \mathbf{\mu}(v_i)\|_2^2]}
\end{equation}
\vspace{-2mm}

\noindent where $\mathbf{b}^{\left(l\right)}_i$ is the $i$-th row of $\mathbf{B}^{\left(l\right)}$, and $\mathbf{\mu}(v_i)$ is the centroid of the bias residuals of all nodes belonging to the demographic group of node $v_i$.
\end{definition}
\vspace{-1mm} \noindent With the distance-based bias, we furthermore show in Theorem \ref{thm:bias_amplification} that the expected squared distance would shrink after message passing in any $l$-th hidden layer, which is equivalent to the amplification of the distance-based bias. Proof is deferred to Appendix ~\ref{appendix:GCN_analyse}.

\begin{figure*}
    \centering
    \includegraphics[width=.85\linewidth]{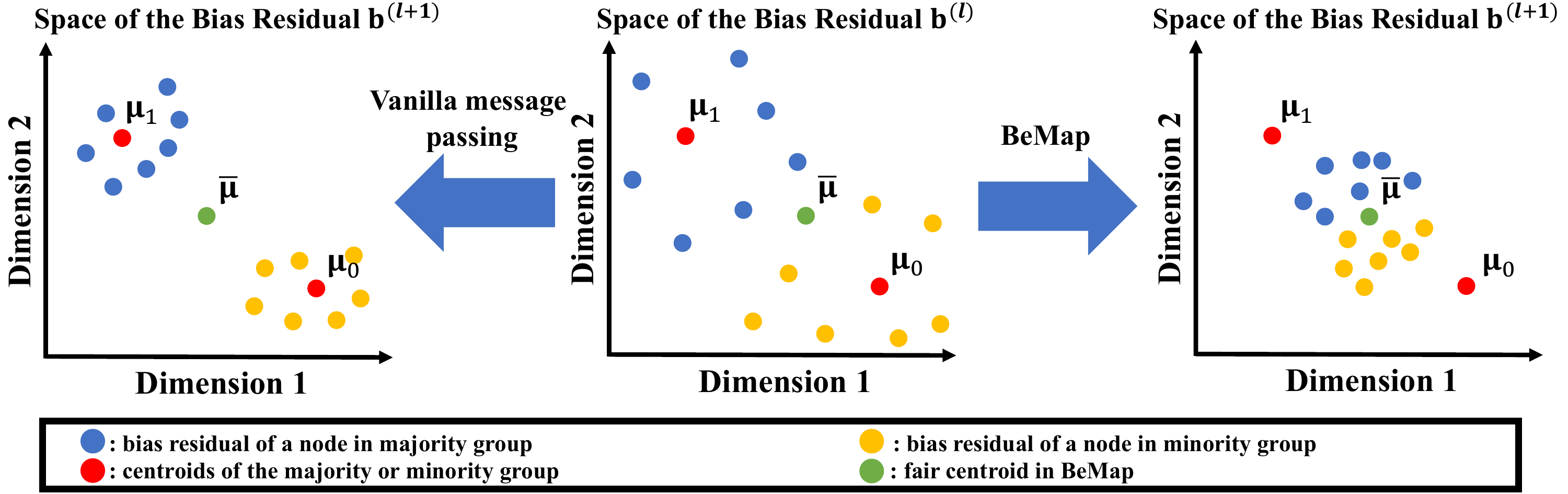}
    \caption{An illustrative example of \method. After \method, the bias residuals will move towards the fair centroid (the green point), whereas, after the vanilla message passing in GCN, they will move towards the centroid of the majority group and the minority group (the red points).}
    \vspace{-6mm}
    \label{fig:bemap_example}
\end{figure*}

\begin{theorem}\label{thm:bias_amplification}
(\textbf{Distance shrinkage}) Suppose we have an $L$-layer linear GCN with row normalization and an input graph $\mathcal{G} = \{\mathcal{V}, \mathbf{A}, \mathbf{X}\}$.
In the $l$-th hidden layer, let \smash{$\mathbf{b}_i^{\left(l\right)}$} and \smash{$d_i$} be the biased embedding and the degree of any \smash{$v_i \in \mathcal{V}$}, respectively. For any $l$-th hidden layer, we have 
\vspace{-1mm}
\begin{equation} \label{eq:dist_shrink_gcn}
\begin{aligned}
    \mathbb{E}_{v_i}[\|\mathbf{b}^{\left(l+1\right)}_i - \mathbf{\mu}(v_i)\|_2^2] &= \mathbb{E}_{v_i}[\frac{1}{d_i}] \mathbb{E}_{v_i}[\|\mathbf{b}^{\left(l\right)}_i - \mathbf{\mu}(v_i)\|_2^2],
\end{aligned}
\end{equation}
which means bias amplification after message passing $\left(\text{bias}^{\left(l\right)}(\mathcal{G}, \mathbf{B}^{\left(l+1\right)}, s) < \text{bias}^{\left(l\right)}(\mathcal{G}, \mathbf{B}^{\left(l\right)}, s) \right)$.
\end{theorem}

\vspace{-1mm}
Theorem \ref{thm:bias_amplification} demonstrates that message passing drives the convergence of bias residuals towards their centroids, thereby magnifying the differences in the distributions of bias residuals of different demographic groups. Moreover, as the expectation of node degrees increases, the distance shrinks at a faster rate. This distance shrinkage exacerbates distance-based biases, causing unfair predictions.


\vspace{-4mm}
\section{\method: A Fair Message Passing Schema}\label{sec:method}
\vspace{-1mm}
In this section, we first present how to avoid bias amplification in message passing and then propose a fair message passing method named \method. 

\vspace{-3mm}
\subsection{Fair Message Passing}
\vspace{-2mm}
Motivated by the observation in Theorem 1 that bias amplification  can lead to unfairness, we put forward the idea that a fair message passing schema should possess the capability to automatically diminish such differences of bias residuals across various demographic groups. Hence, a fair message passing schema entails two primary objectives: (1) \textbf{centroid consistency}. The centroid of bias residuals becomes the same across different demographic groups after fair message passing; (2) \textbf{distance shrinkage}. The distance between bias residuals and their centroids will shrink through message passing. An illustrative example is presented in Figure \ref{fig:bemap_example}.

To achieve our first objective \textbf{centroid consistency}, as shown in Lemma~\ref{lm:sufficiency_mean_consist}, we have to make sure the proportion of neighbors from each demographic group is similar.

\begin{lemma}\label{lm:sufficiency_mean_consist}
(\textbf{Sufficient condition for centroid consistency}) Suppose Assumption~\ref{asmp:random_graph} and Assumption~\ref{asmp:exist_fair_node_emb} hold, and we are given an input graph $\mathcal{G} = \left\{\mathcal{V}, \mathbf{A}, \mathbf{X}\right\}$ and an $L$-layer linear GCN with row normalization.
For any node $v_i\in\mathcal{V}$, let ${\widehat N}_0\left(v_i\right)$ and ${\widehat N}_1\left(v_i\right)$ be the number of neighbors in $\mathcal{\widehat N}\left(v_i\right) = \mathcal{N}\left(v_i\right) \cup \left\{v_i\right\}$ that belong to the minority group and majority group, respectively. In the $l$-th hidden layer, the centroid of the bias residual $\mathbf{b}_i^{(l)}$ keeps the same for any $v_i\in\mathcal{V}$ when 
\vspace{-2mm}
\begin{equation}\label{eq:sufficiency_fair_prototype}
    \frac{{\widehat N}_0\left(v_i\right)}{{\widehat N}_1\left(v_i\right)} = \frac{r_0}{r_1}, \quad\forall v_i\in\mathcal{V}
\end{equation}
\noindent 
where $r_s, s \in \left\{0, 1\right\}$ is the ratio of neighbors from demographic group $s$, and $r_0 + r_1 = 1$.
\end{lemma}

\vspace{-1mm}
Lemma~\ref{lm:sufficiency_mean_consist} states that, as long as the self-augmented neighborhood of any node in the graph has the same ratio of neighbors from each demographic group, the bias residuals would be centered around the same centroid. Actually, the new fair centroid, which is denoted as $\mathbf{\bar \mu}$, can be expressed as a linear combination of the original centroid of each demographic group, \ie, $\mathbf{\bar \mu} = r_0 \mathbf{\mu}_0 + r_1 \mathbf{\mu}_1$. Therefore, we can change the position of the fair centroid $\mathbf{\bar \mu}$ by changing the ratio of each demographic group. However, if edge deletion is used to control the ratio for every node, then the ratio $r_s$ are not supposed to be some extreme values, \ie, $0$ or $1$. It is because that those extreme values will result in the removal of a large number of edges, thus leading to a severe degradation in the utility of the model. To ensure the balance between utility and fairness, we set the ratio of the minority group and majority group to be the same, \ie, $r_0 = r_1 = \frac{1}{2}$.

To achieve our second objective \textbf{distance shrinkage}, as shown in Theorem~\ref{thm:bemap}, by increasing the balanced self-augmented neighborhood to  sufficiently large, the expected squared distance between the bias residual and their fair centroids $\mathbf{\bar \mu}$  will shrink after message passing via self-augmented neighborhood. Proof is deferred to Appendix~\ref{appendix:BeMap_analyse}.

\vspace{-1mm}
\begin{theorem}\label{thm:bemap}
    (\textbf{Fair message passing}) Suppose Lemma~\ref{lm:sufficiency_mean_consist} holds. As long as ${\widehat N}_0\left(v_i\right) + {\widehat N}_1\left(v_i\right)\geq 4$, $\forall v_i\in\mathcal{V}$, the expected squared distance between the bias residual $\mathbf{b}_i^{(l)}$ and the fair centroid $\mathbf{\bar \mu}$ will shrink after message passing. Mathematically, we have 
    \vspace{-1mm}
    \begin{equation} \label{eq:fair_dist_shrink}
        \mathbb{E}_{v_i}[\|\mathbf{b}_i^{(l+1)} - \mathbf{\bar \mu}\|_2^2] < \mathbb{E}_{v_i}[\|\mathbf{b}_i^{(l)} - \mathbf{\bar \mu}\|_2^2], \quad l\in\left\{1, \ldots, L\right\}
    \end{equation}
\end{theorem}

\vspace{-4mm}
\subsection{\method~Algorithm}
\vspace{-2mm}
To achieve \textbf{centroid consistency} and \textbf{distance shrinkage}, we propose a fair message passing algorithm named \method, whose key idea is to perform message passing on a sufficiently large and balanced self-augmented neighborhood (which we call fair neighborhood). To obtain the fair neighborhood, we consider sampling (i.e., edge deletion) over the original self-augmented neighborhood in \method. Note that other types of neighborhood augmentation techniques (e.g., edge addition, edge rewiring) could also be applied to obtain the fair neighborhood, which we leave for future work. Hereafter, \method~is referred to as message passing over the sampled fair neighborhood.

In practice, there are two main challenges in obtaining the fair neighborhood in \method~ due to the long-tailed in many real-world networks. First, there could exist some node $v_i$ such that ${\widehat N}_0\left(v_i\right) = {\widehat N}_1\left(v_i\right)$ and ${\widehat N}_0\left(v_i\right) + {\widehat N}_1\left(v_i\right) \geq 4$ cannot be satisfied simultaneously. In this case, we apply the following two empirical remedies.
\vspace{-2mm}
\begin{itemize}[
    align=left,
    leftmargin=2em,
    itemindent=0pt,
    labelsep=0pt,
    labelwidth=1em,
]
    \item If all neighbors in $\mathcal{\widehat N}\left(v_i\right)$ of a node $v_i$ belong to one demographic group, we sample a subset of $k$ nodes where $k = \max\left\{4, \beta|\mathcal{N}\left(v_i\right)|\right\}$, $|\mathcal{N}\left(v_i\right)|$ is the cardinality of $\mathcal{N}\left(v_i\right)$ and $\beta$ is a hyperparameter. The intuition is that when the local neighborhood is sufficiently large to satisfy Theorem \ref{thm:bemap} (\ie, $k\ge 4$), decreasing the number of neighbors (i.e., the degree $d_i$ of node $v_i$) helps reduce the difference between the expected squared distances before and after message passing as shown in Eq.~\eqref{eq:dist_shrink_gcn}. Thus, it helps decelerate the bias residuals moving towards the original centroid of the demographic group of $v_i$. 
    \vspace{-2mm}
    \item Otherwise, we keep the sampled neighborhood to be balanced, i.e., ${\widehat N}_0\left(v_i\right) = {\widehat N}_1\left(v_i\right)$, by sampling over $\mathcal{N}_0\left(v_i\right)$ if ${\widehat N}_0\left(v_i\right) > {\widehat N}_1\left(v_i\right)$ or $\mathcal{N}_1\left(v_i\right)$ if ${\widehat N}_0\left(v_i\right) < {\widehat N}_1\left(v_i\right)$. By balancing the sampled neighborhood, it helps shrink the bias residual to converge into the fair centroid $\mathbf{\bar \mu}$. 
\end{itemize}

\vspace{-1mm}
Second, there might exist some node $v_i$ such that its neighbors within $L$ hops could contain node(s) whose neighbors always belong to only one demographic group. Note that the receptive field of the $l$-th hidden layer in an $L$-layer GCN is the $l$-hop neighborhood of a node. Then, for such node $v_i$, it is hard to maintain a balanced $l$-hop neighborhood if we apply uniform sampling, so as to satisfy Lemma~\ref{lm:sufficiency_mean_consist}. To alleviate this issue, we propose \textit{balance-aware sampling}. Its key idea is to adjust the sampling probability based on the difference between the numbers of neighbors in the majority group and in the minority group. Specifically, for any node $v_i\in\mathcal{V}$, we define the balance score as 
\begin{equation}\label{eq:balance_score}
    \text{balance}_i = \frac{1}{|{\widetilde N}_0\left(v_i\right) - {\widetilde N}_1\left(v_i\right)| + \delta}
\end{equation}
\noindent where $\delta$ is a hyperparameter to avoid division by zero and ${\widetilde N}_0\left(v_i\right)$ as well as ${\widetilde N}_1\left(v_i\right)$ are the numbers of all neighbors within $L$ hops that belong to the minority group and majority group, respectively. Then in the balance-aware sampling, for any node $v_i\in\mathcal{V}$, the sampling probability of node $v_j\in\mathcal{N}\left(v_i\right)$ is
\begin{equation}\label{eq:balance_aware_sampling}
    P\left(v_j | v_i\right) = \frac{\text{balance}_j}{\sum_{v_k\in\mathcal{N}\left(v_i\right)} \text{balance}_k}
\end{equation}
In general, \method~includes three key steps: pre-processing, balance-aware sampling, and fair message passing. First, during pre-processing, we precompute the sampling probability based on Eq. \eqref{eq:balance_aware_sampling}. Then, during the every epoch of training, we first generate the fair neighborhood using the balance-aware sampling, and then aggregate neighbors' information on the generated fair subgraph, which is called fair message passing. Finally, we update the model parameters with back-propagation. The detailed workflow and the extension to non-binary sensitive attribute of \method~is presented in Appendix \ref{sec:alg_appendix} and Appendix \ref{sec:nonbinary_appendix}, respectively.

\vspace{-4mm}
\section{Experiments}\label{sec:experiments}
\vspace{-1mm}
\subsection{Experimental Settings}
\vspace{-1mm}
\noindent\textbf{Datasets.} We conduct experiments on 4 real-world datasets, including \textit{Pokec-z}~\cite{nguyen2021graphdta}, \textit{Pokec-n}~\cite{nguyen2021graphdta}, \textit{NBA}~\cite{dai2021say}, \textit{Credit}~\cite{dong2022edits}, and \textit{Recidivism}~\cite{agarwal2021towards}. For each dataset, we use the same 50\%/25\%/25\% splits for train/validation/test sets. Detailed descriptions of the datasets are provided in Appendix \ref{sec:appendix_datasets}.

\vspace{-1mm}
\noindent \textbf{Baseline Methods.} We compare \method~with graph neural networks with and without fairness considerations. Graph neural networks without fairness considerations include GCN~\cite{kipf2016semi} and GraphSAGE~\cite{hamilton2017inductive}. Fair graph neural networks for comparison include \textit{FairGNN}~\cite{dai2021say}, \textit{EDITS}~\cite{dong2022edits}, \textit{FairDrop}~\cite{spinelli2021biased}, \textit{NIFTY}~\cite{agarwal2021towards} and \textit{FMP}~\cite{jiang2022fmp}. Descriptions of baseline methods are in Appendix~\ref{sec:appendix_baseline_methods}. For \method, \method~(sym) uses GCN with row normalization, \method~(sym) uses GCN with symmetric normalization, and \method~(gat) uses GAT. 

\vspace{-1mm}
\noindent\textbf{Parameter Settings.} Unless otherwise specified, we use default hyperparameter settings in the released code of corresponding publications. For \method, we set the learning rate as $1e-3$, weight decay as $1e-5$, $\beta$ as $\frac{1}{4}$, and $\delta$ as $1$. We set the backbone model of \method~as a $2$-layer GCN with $128$ hidden dimensions and the optimizer as Adam. To be consistent with the number of hidden layers, for each node, all neighbors within $2$ hops are used to calculate the balance score in Eq.~\eqref{eq:balance_score}.

\vspace{-1mm}
\noindent\textbf{Metrics.} We consider the task of semi-supervised node classification. To measure the utility, we use the classification accuracy (ACC) and the area under the receiver operating characteristic curve (AUC). Regarding fairness, we use $\spmetric$ and $\eometric$ mentioned in Section~\ref{sec:prelim}. For ACC and AUC, the higher the better; while for $\spmetric$ and $\eometric$, the lower the better. 

\vspace{-3mm}
\subsection{Experimental Results}
\vspace{-1mm}



\begin{table*}[t]
    \centering
    \caption{Main results on semi-supervised node classification. Higher is better ($\uparrow$) for ACC and AUC (white). Lower is better ($\downarrow$) for $\spmetric$ and $\eometric$ (gray). Bold font indicates the best performance for fair graph neural networks, and underlined number indicates the second best.}
    \vspace{-2mm}
    \resizebox{\linewidth}{!}{
    \begin{tabular}{c|c|c|c|c|c|c|c|c}
        \hline
        \multirow{2}{*}{\textbf{Methods}} & \multicolumn{4}{c}{\textbf{Pokec-z}} & \multicolumn{4}{c}{\textbf{NBA}} \\
        \cline{2-9}
        & ACC(\%) $\uparrow$ & AUC(\%) $\uparrow$ & \cellcolor{black!15} $\spmetric$(\%) $\downarrow$ & \cellcolor{black!15} $\eometric$(\%) $\downarrow$ & ACC(\%) $\uparrow$ & AUC(\%) $\uparrow$ & \cellcolor{black!15}$\spmetric$(\%) $\downarrow$ & \cellcolor{black!15} $\eometric$(\%) $\downarrow$ \\
        \hline
        GCN & $70.62 \pm 0.22$ & $76.41 \pm 0.61$ & \cellcolor{black!15} $8.86 \pm 2.32$ & \cellcolor{black!15} $7.81 \pm 1.99$ & $72.16 \pm 0.46$ & $78.45 \pm 0.25$ & \cellcolor{black!15} $4.00 \pm 2.15$ & \cellcolor{black!15} $13.07 \pm 3.34$ \\
        GraphSAGE & $70.27 \pm 0.32$ & $75.73 \pm 0.30$ & \cellcolor{black!15} $7.11 \pm 2.69$ & \cellcolor{black!15}$6.97 \pm 2.65$ & $75.21 \pm 0.86$ & $77.83 \pm 2.00$ & \cellcolor{black!15} $7.81 \pm 3.45$ & \cellcolor{black!15} $8.87 \pm 5.58$ \\
        GAT & 64.24 $\pm$ 0.57 & 69.48 $\pm$ 0.53 &\cellcolor{black!15} 11.56 $\pm$ 0.94 &\cellcolor{black!15} 12.40 $\pm$ 1.85 & 72.14 $\pm$ 2.34 & 76.94 $\pm$ 0.82 &\cellcolor{black!15}\cellcolor{black!15} 5.97 $\pm$ 2.89 &\cellcolor{black!15} 12.07 $\pm$ 2.47 \\
        \hline
        FairGNN & $65.74 \pm 2.49$ & $72.54 \pm 4.21$ & \cellcolor{black!15} $4.31 \pm 0.80$ & \cellcolor{black!15} $4.34 \pm 1.22$ & $\bf 72.39 \pm 0.46$ & $77.74 \pm 1.24$ & \cellcolor{black!15} $3.96 \pm 1.81$ & \cellcolor{black!15} $4.94 \pm 2.35$ \\
        EDITS & $67.25 \pm 0.61$ & $73.05 \pm 0.57$ & \cellcolor{black!15} $10.36 \pm 1.20$ & \cellcolor{black!15} $9.00 \pm 1.34$ & $66.19 \pm 0.75$ & $69.94 \pm 0.72$ & \cellcolor{black!15} $18.15 \pm 3.64$ & \cellcolor{black!15} $13.19 \pm 3.48$ \\
        FairDrop & $67.45 \pm 0.80$ & $73.77 \pm 0.50$ & \cellcolor{black!15} $9.46 \pm 2.06$ & \cellcolor{black!15} $7.91 \pm 1.59$ & $70.81 \pm 0.63$ & $77.05 \pm 0.46$ & \cellcolor{black!15} $5.42 \pm 1.59$ & \cellcolor{black!15} $7.30 \pm 1.34$ \\
        NIFTY & $67.55 \pm 0.79$ & \underline{$73.87 \pm 0.23$} & \cellcolor{black!15} $8.83 \pm 1.60$ & \cellcolor{black!15} $7.00 \pm 1.70$ & $60.84 \pm 3.32$ & $65.94 \pm 1.30$ & \cellcolor{black!15} $10.03 \pm 5.46$ & \cellcolor{black!15} $5.70 \pm 3.21$ \\
        FMP & $\bf 72.05 \pm 0.42$ & $\bf 80.50 \pm 0.11$ & \cellcolor{black!15} $5.03 \pm 1.22$ & \cellcolor{black!15} $1.72 \pm 0.64$ & $67.57 \pm 1.58$ & $77.73 \pm 0.35$ & \cellcolor{black!15} $31.41 \pm 1.11$ & \cellcolor{black!15} $27.17 \pm 3.19$ \\
        \hline
        \method~(row) & $68.88 \pm 0.30$ & $72.34 \pm 0.44$ & \cellcolor{black!15} $\bf 0.74 \pm 0.52$ & \cellcolor{black!15} $1.55 \pm 0.25$ & $67.84 \pm 0.64$ & $\bf 78.87 \pm 0.17$ & \cellcolor{black!15} \underline{$3.91 \pm 1.28$ } & \cellcolor{black!15} \underline{$4.08 \pm 2.15$} \\
        \method~(sym) & \underline{$68.94 \pm 0.46$} & $73.01 \pm 0.29$ & \cellcolor{black!15} $1.45 \pm 0.40$ & \cellcolor{black!15} $\bf 1.03 \pm 0.42$ & $65.37 \pm 1.77$ & \underline{$78.76 \pm 0.62$} & \cellcolor{black!15} $\bf 3.54 \pm 0.97$ & \cellcolor{black!15} $\bf 3.81 \pm 0.98$ \\
        \method~(gat) & 66.89 $\pm$ 1.26 & 71.42 $\pm$ 0.96 & \cellcolor{black!15} \underline{0.97 $\pm$ 0.22} &  \cellcolor{black!15} \underline{1.29 $\pm$ 0.78} & \underline{71.99 $\pm$ 1.10} & 77.36 $\pm$ 0.01 &\cellcolor{black!15} 4.01 $\pm$ 3.26 & \cellcolor{black!15}5.08 $\pm$ 2.8 \\
        \hline
    \end{tabular}
    }
    \resizebox{\linewidth}{!}{\begin{tabular}{c|c|c|c|c|c|c|c|c}
        \hline
        \multirow{2}{*}{\textbf{Methods}} & \multicolumn{4}{c}{\textbf{Recidivism}} & \multicolumn{4}{c}{\textbf{Credit}} \\
        \cline{2-9}
        & ACC(\%) $\uparrow$ & AUC(\%) $\uparrow$ & \cellcolor{black!15} $\spmetric$(\%) $\downarrow$ & \cellcolor{black!15} $\eometric$(\%) $\downarrow$ & ACC(\%) $\uparrow$ & AUC(\%) $\uparrow$ & \cellcolor{black!15}$\spmetric$(\%) $\downarrow$ & \cellcolor{black!15} $\eometric$(\%) $\downarrow$ \\
        \hline
        GCN & $85.89 \pm 0.18$ & $88.74 \pm 0.23$ & \cellcolor{black!15} $8.51 \pm 0.48$ & \cellcolor{black!15} $5.75 \pm 1.08$ &  $75.80 \pm 0.92$ & $73.83 \pm 1.50$ & \cellcolor{black!15} $17.92 \pm 0.50$ & \cellcolor{black!15} $15.41 \pm 0.77$ \\
        GraphSAGE & $85.00 \pm 0.52$ & $89.41 \pm 0.36$ & \cellcolor{black!15} $8.99 \pm 0.37$ & \cellcolor{black!15} $6.14 \pm 0.54$ & $74.25 \pm 0.25$ & $74.35 \pm 0.12$ & \cellcolor{black!15} $13.82 \pm 0.84$ & \cellcolor{black!15} $11.62 \pm 0.87$ \\
        GAT & 88.66 $\pm$ 1.08 & 92.31 $\pm$ 0.62 &  \cellcolor{black!15} 7.45 $\pm$ 0.48 &  \cellcolor{black!15} 4.99 $\pm$ 0.67 & 70.61 $\pm$ 2.30 & 73.78 $\pm$ 0.35&  \cellcolor{black!15} 10.85 $\pm$ 0.62 &  \cellcolor{black!15} 8.84 $\pm$0.56\\
        \hline
        FairGNN & $69.54 \pm 5.70$ & $80.79 \pm 5.33$ & \cellcolor{black!15} $6.61 \pm 2.29$ & \cellcolor{black!15} $4.75 \pm 3.50$ & \underline{$75.34 \pm 1.18$} & $70.79 \pm 2.76$ & \cellcolor{black!15} $11.23 \pm 4.69$ & \cellcolor{black!15} $8.95 \pm 2.76$ \\
        EDITS & \underline{$83.88 \pm 0.84$} & \underline{$86.82 \pm 0.40$} & \cellcolor{black!15} $7.63 \pm 0.57$ & \cellcolor{black!15} $5.06 \pm 0.44$ & $73.49 \pm 0.03$ & $\bf 73.52 \pm 0.10$ & \cellcolor{black!15} $13.33 \pm 0.15$ & \cellcolor{black!15} $10.93 \pm 0.06$ \\
        FairDrop & $\bf 91.81 \pm 0.36$ & $\bf 92.17 \pm 0.86$ & \cellcolor{black!15} $6.80 \pm 0.22$ & \cellcolor{black!15} $3.30 \pm 0.13$ & $68.41 \pm 9.20$ & $70.76 \pm 4.11$ & \cellcolor{black!15} $14.23 \pm 2.24$ & \cellcolor{black!15} $12.01 \pm 2.14$ \\
        NIFTY & $83.43 \pm 0.83$ & $84.56 \pm 0.33$ & \cellcolor{black!15} $4.75 \pm 0.92$ & \cellcolor{black!15} $4.04 \pm 1.46$ & $73.48 \pm 0.04$ & $72.33 \pm 0.01$ & \cellcolor{black!15} $11.80 \pm 0.09$ & \cellcolor{black!15} $9.51 \pm 0.08$ \\
        FMP & $56.80 \pm 10.83$ & $61.79 \pm 4.87$ & \cellcolor{black!15} $22.43 \pm 9.72$ & \cellcolor{black!15} $17.50 \pm 8.87$ & $74.37 \pm 0.21$ & $72.92 \pm 0.14$ & \cellcolor{black!15} $14.07 \pm 0.78$ & \cellcolor{black!15} $11.87 \pm 0.72$ \\
        \hline
        \method~(row) & $77.61 \pm 0.33$ & $81.11 \pm 1.44$ & \cellcolor{black!15} \underline{$2.87 \pm 0.46$} & \cellcolor{black!15} $\bf 2.09 \pm 0.38$ & $71.46 \pm 0.74$ & $69.41 \pm 0.71$ & \cellcolor{black!15} \underline{$\10.51 \pm 0.23$} & \cellcolor{black!15} $10.87 \pm 1.55$ \\
        \method~(sym) & $77.66 \pm 0.35$ & $80.77 \pm 1.02$ & \cellcolor{black!15} $\bf 1.76 \pm 0.14$ & \cellcolor{black!15} \underline{$2.68 \pm 0.12$} & $72.55 \pm 0.28$ & \underline{$72.98 \pm 0.66$} & \cellcolor{black!15} $10.74 \pm 0.46$ & \cellcolor{black!15} \underline{$ 8.36 \pm 0.46$} \\
        \method~(gat) & 71.83 $\pm$ 1.23 & 71.97 $\pm$ 1.28 & \cellcolor{black!15}3.67 $\pm$ 1.22 & \cellcolor{black!15}3.60 $\pm$ 1.78 & $\bf 76.72 \pm 0.77$ & 68.18 $\pm$ 0.50 &\cellcolor{black!15} $\bf 7.92 \pm 1.03$ &\cellcolor{black!15} $\bf 6.22 \pm 0.57$ \\
        \hline
    \end{tabular}
    }
    \label{tbl:effectiveness}
    \vspace{-4mm}
\end{table*}

\noindent \textbf{Main Results.} 
The main evaluation results on the utility (ACC and AUC) and fairness ($\eometric$ and $\spmetric$) are presented in Table~\ref{tbl:effectiveness}. Similar evaluation results on Pokec-n are presented in Table~\ref{tab:pokec_n_results}. Regarding fairness, our proposed \method~is the only method that can consistently mitigate bias (i.e., a smaller value of $\eometric$ and $\spmetric$ than the vanilla GCN and GAT) for all datasets. 
Moreover, compared with the vanilla GCN and GAT, all the variants of \method~ could effectively reduce $\eometric$ and $\spmetric$ to a low degree. For example, on the Pokec-z dataset, $\eometric$($\spmetric$) are reduced to 8.35\% (19.84\%), 24.75\% (14.77\%) and 8.39\% (10.40\%) of original values for \method~(row), \method~(sym), and \method~(gat), respectively. More detailed statistics are exhibited in Table \ref{tab:percentage_of_deline_appendix}. At the same time, \method~ also achieves comparable performance in terms of the utility (ACC and AUC). 
This is because \method~ generates a new balanced graph for each epoch, analogous to data augmentation, which prevents the graph neural network from overfitting. For example, on NBA, the AUC scores of \method~(sym) and \method~(row) are $78.87\%$ and $78.76\%$, respectively, both of which are higher than the AUC of vanilla GCN and GraphSAGE ($78.45\%$ and $77.83\%$, respectively).
In short words, \method~could achieve a good balance between mitigating the bias and maintaining the classification accuracy.


\begin{table*}[t]
    \centering
    \caption{Ablation study of different sampling strategy. Higher is better ($\uparrow$) for ACC and AUC (white). Lower is better ($\downarrow$) for $\spmetric$ and $\eometric$ (gray). Bold font indicates the best performance for fair graph neural networks, and underlined number indicates the second best.}
    \vspace{-2mm}
    \resizebox{\linewidth}{!}{
    \begin{tabular}{c|c|c|c|c|c|c|c|c}
        \hline
        \textbf{Sampling} & \multicolumn{4}{c}{\textbf{Pokec-z}} & \multicolumn{4}{c}{\textbf{NBA}} \\
        \cline{2-9}
        \textbf{Methods} & ACC(\%) $\uparrow$ & AUC(\%) $\uparrow$ & \cellcolor{black!15} $\spmetric$(\%) $\downarrow$ & \cellcolor{black!15} $\eometric$(\%) $\downarrow$ & ACC(\%) $\uparrow$ & AUC(\%) $\uparrow$ & \cellcolor{black!15}$\spmetric$(\%) $\downarrow$ & \cellcolor{black!15} $\eometric$(\%) $\downarrow$ \\
        \hline
        Uniform & \underline{$69.65 \pm 0.23$} & \underline{$73.50 \pm 0.24$} & \cellcolor{black!15} \underline{$4.98 \pm 0.78$} & \cellcolor{black!15} \underline{$2.87 \pm 1.19$} & $\bf 72.67 \pm 1.58$ & \underline{$78.48 \pm 0.42$} & \cellcolor{black!15} \underline{$3.70 \pm 1.18$} & \cellcolor{black!15} \underline{$12.02 \pm 2.97$} \\
        Degree & $\bf 70.00 \pm 0.36$ & $\bf 73.98 \pm 0.37$ & \cellcolor{black!15} $5.85 \pm 1.46$ & \cellcolor{black!15} $3.40 \pm 1.84$ & \underline{$71.64 \pm 1.41$} & \underline{$78.37 \pm 0.32$} & \cellcolor{black!15} $4.25 \pm 1.60$ & \cellcolor{black!15} $14.88 \pm 4.60$ \\
        Balance-aware & $68.94 \pm 0.46$ & $73.01\pm0.29$ & \cellcolor{black!15} $\bf 1.45 \pm 0.40$ & \cellcolor{black!15} $\bf 1.03 \pm 0.42$ & $65.37 \pm 1.77$ & $\bf 78.76 \pm 0.62$ & \cellcolor{black!15} $\bf 3.54 \pm 0.97$ & \cellcolor{black!15} $\bf 3.81 \pm 0.98$ \\
        \hline
    \end{tabular}
    }
    \resizebox{\linewidth}{!}{\begin{tabular}{c|c|c|c|c|c|c|c|c}
        \hline
        \textbf{Sampling} & \multicolumn{4}{c}{\textbf{Recidivism}} & \multicolumn{4}{c}{\textbf{Credit}} \\
        \cline{2-9}
        \textbf{Methods} & ACC(\%) $\uparrow$ & AUC(\%) $\uparrow$ & \cellcolor{black!15} $\spmetric$(\%) $\downarrow$ & \cellcolor{black!15} $\eometric$(\%) $\downarrow$ & ACC(\%) $\uparrow$ & AUC(\%) $\uparrow$ & \cellcolor{black!15}$\spmetric$(\%) $\downarrow$ & \cellcolor{black!15} $\eometric$(\%) $\downarrow$ \\
        \hline
        Uniform & \underline{$81.80 \pm 0.93$} & $76.58 \pm 0.31$ & \cellcolor{black!15} \underline{$4.40 \pm 0.15$} & \cellcolor{black!15} $3.29 \pm 0.45$ & \underline{$75.04 \pm 0.40$} & $\bf 73.66 \pm 0.14$ & \cellcolor{black!15} \underline{$14.34 \pm 0.64$} & \cellcolor{black!15} \underline{$11.70 \pm 0.54$} \\
        Degree & $\bf 85.06 \pm 1.16$ & $\bf 87.12 \pm 0.52$ & \cellcolor{black!15} $6.13 \pm 0.13$ & \cellcolor{black!15} \underline{$2.78 \pm 0.90$} & $\bf 75.82 \pm 0.45$ & \underline{$73.64 \pm 0.06$} & \cellcolor{black!15} $15.45 \pm 0.82$ & \cellcolor{black!15} $12.59 \pm 0.84$ \\
        Balance-aware & $77.66 \pm 0.35$ & \underline{$80.77 \pm 1.02$} & \cellcolor{black!15} $\bf 1.76 \pm 0.14$ & \cellcolor{black!15} $\bf 2.68 \pm 0.12$ & $72.55 \pm 0.28$ & $72.98\pm 0.66$ & \cellcolor{black!15} $\bf 10.74 \pm 0.46$ & \cellcolor{black!15} $\bf 8.36 \pm 0.46$ \\
        \hline
    \end{tabular}
    }
    \label{tbl:ablation}
    \vspace{-5mm}
\end{table*}

\vspace{-1mm}
\noindent \textbf{Ablation Study.} 
To evaluate the effectiveness of the balance-aware sampling, we compare it with two other heuristic sampling strategies: (1) \textit{uniform sampling} (Uniform), which assigns the same probability to all neighbors of a node, and (2) \textit{degree-based sampling} (Degree), which sets the probability of node $v_i$ in the neighborhood of node $u$ as $P\left(v_i \mid u\right)\propto d_i^{0.75}, \forall v_i\in \mathcal{N}(u)$~\cite{tang2015line}. From Table~\ref{tbl:ablation}, we can see that the balance-aware sampling achieves the lowest $\eometric$ and $\spmetric$ on all datasets and maintains a comparable classification accuracy, which demonstrates the superiority of the balance-aware sampling in balancing the utility and fairness.

\vspace{-4mm}
\section{Conclusion}
\vspace{-2mm}
In this paper, we study bias amplification in message passing and fair message passing. We empirically and theoretically prove that message passing amplifies the bias as long as the numbers of neighbors from different demographic groups for each node are unbalanced. Guided by our analyses, we propose \method, which relies on a balance-aware sampling strategy to generate a fair neighborhood among different demographic groups. Then, \method~performs message passing over the generated fair neighborhood. Extensive evaluations on the real-world datasets demonstrate the effectiveness of our proposed method in mitigating bias while maintaining utility.

\section*{Acknowledgments}
This work is supported by NSF (1947135, 
2134079, 
1939725, 
2316233, 
and 2324770), 
DARPA (HR001121C0165), DHS (17STQAC00001-07-00), NIFA (2020-67021-32799) and ARO (W911NF2110088).

\bibliographystyle{unsrtnat}
\bibliography{reference}
\appendix

\section{Key Symbols of \method}

\begin{table}[htbp]
    \centering
    \caption{Table of key symbols in the paper.}
    \begin{tabular}{l|l}
        \toprule
        \textbf{Symbol} & \textbf{Definition} \\
        \midrule
        $\mathcal{V}$ & The set of nodes \\
        $\mathcal{V}_s$ & The set of nodes with the sensitive attribute of $s$ \\
        $\mathcal{N}$ & The set of $1$-hop neighbors \\
        $\mathcal{\widehat{\mathcal{N}}}$ & The set of $1$-hop neighbors and the node itself \\
        $\mathcal{\mathcal{N}}^k$ & The set of $k$-hop neighbors \\
        $\mathcal{\mathcal{N}}_s$ & The set of $1$-hop neighbors with the sensitive attribute of $s$ \\
        \midrule
        $\mathbf{A}$ & The adjacency matrix \\
        $\mathbf{\Tilde{A}}$ & The row-normalized adjacency matrix \\
        $\mathbf{X}$ & The node feature matrix \\
        $\mathbf{H}^{(k+1)}$ & The node representation matrix of the $k$-th hidden layer \\
        $\mathbf{D}$ & The degree matrix of nodes\\
        $\mathbf{W}^{(k)}$ & The weight matrix of the $k$-th hidden layer \\
        $\mathbf{T}$ & The fair embedding matrix \\
        $\mathbf{B}$ & The bias residual matrix \\
        \midrule
        $\mathbf{x}_i$ & The node feature of the node $i$ \\
        $\mathbf{h}_i^{(k)}$ & The node representation of the node $i$ in the $k$-th hidden layer \\
        $\mathbf{t}_i^{(k)}$ & The fair embedding of the node $i$ in the $k$-th hidden layer \\
        $\mathbf{b}_i^{(k)}$ & The bias residual of the node $i$ in the $k$-th hidden layer \\
        $\mathbf{b}_{i, s}^{(k)}$ & The bias residual of the node $i$ with the sensitive attribute of $s$ in the $k$-th hidden layer \\
        $\mu$ & The centroid of all the nodes \\
        $\mu_s$ & The centroid of the nodes with the sensitive attribute of $s$ \\
        \midrule
        $d_i$ & The degree of the node $i$ \\
        $r_s$ & The ratio of the neighbors with the sensitive attribute of $s$ in \method \\
        
        \bottomrule
    \end{tabular}
    \label{tab:notation_appendix}
\end{table}

\section{Proof of Lemma~\ref{lm:linear_node_embed}} \label{appendix:linear_combine}

Given a $L$-layer linear GCN, the weight matrix in the $l$-th hidden layer is represented as $\mathbf{W}^{\left(l\right)}$, and the input node features and output node features for all nodes are denoted as $\mathbf{H}^{\left(l\right)}$ and $\mathbf{H}^{\left(l+1\right)}$. Then the output node features can be calculated as $\mathbf{H}^{\left(l+1\right)} = \widetilde{\mathbf{A}} \mathbf{H}^{\left(l\right)} \mathbf{W}^{\left(l\right)}$, where $\widetilde{\mathbf{A}}$ is the normalized adjacency matrix. Therefore, the output node feature of the last hidden layer can be expressed as
\begin{equation} \label{eq:merge_mat_expres}
    \mathbf{Z} = \widetilde{\mathbf{A}}^L \mathbf{H}^{(1)} \mathbf{W}
\end{equation}
where $\mathbf{W} = \mathbf{W}^{(1)} \mathbf{W}^{(2)} \dots \mathbf{W}^{\left(L\right)}$. Then given a particular output node feature $\mathbf{Z}$, if there exists a solution, then a feasible solution of the input node feature is

\begin{equation}
    \mathbf{H}^{(1)} = (\widetilde{\mathbf{A}}^L)^{\dagger} \mathbf{Z} \mathbf{W}^{\dagger}
\end{equation}
where $\mathbf{W}^{\dagger}$ and $(\widetilde{\mathbf{A}}^L)^{\dagger}$ are the Moore–Penrose inverse of $\mathbf{W}$ and $\widetilde{\mathbf{A}}^L$, respectively. 

Based on Assumption~\ref{asmp:exist_fair_node_emb}, the input node feature can be further expressed as
\begin{equation}
    \mathbf{H}^{(1)} = (\widetilde{\mathbf{A}}^L)^{\dagger}  \mathbf{Z}_t \mathbf{W}^{\dagger} + (\widetilde{\mathbf{A}}^L)^{\dagger} \mathbf{Z}_b \mathbf{W}^{\dagger} = \mathbf{T}^{(1)} + \mathbf{B}^{(1)}
\end{equation}
where $\mathbf{T}^{(1)} = (\widetilde{\mathbf{A}}^L)^{\dagger} \mathbf{Z}_t  \mathbf{W}^{\dagger}$ is the fair part of input node features while $\mathbf{B}^{(1)} = (\widetilde{\mathbf{A}}^L)^{\dagger} \mathbf{Z}_b \mathbf{W}^{\dagger}$ is the bias residual which lead to discrimination towards different demographic groups in predictions. Therefore, the input node feature could be linearly separated into two vectors: the fair embeddings $\mathbf{T}^{(1)}$ and the bias residual $\mathbf{B}^{(1)}$.

Then for any  $l$-th hidden layer $\ \in \left\{1, \dots, L\right\}$, the following expression holds.
\begin{equation}
    \begin{aligned}
        \mathbf{H}^{\left(l\right)}
        &= \widetilde{\mathbf{A}} \mathbf{H}^{\left(l-1\right)} \mathbf{W}^{\left(l-1\right)} \\
        & = \widetilde{\mathbf{A}}^{l-1}\mathbf{H}^{\left(1\right)}\mathbf{W}^{\left(1\right)}\dots\mathbf{W}^{\left(l-1\right)} \\
        & = \widetilde{\mathbf{A}}^{l-1}\left(\mathbf{T}^{(1)} + \mathbf{B}^{(1)}\right)\mathbf{W}^{\left(1\right)}\dots\mathbf{W}^{\left(l-1\right)} \\
        & = \left(\widetilde{\mathbf{A}}^{l-1}(\widetilde{\mathbf{A}}^L)^{\dagger} \mathbf{Z}_t\mathbf{W}^{\dagger}  + \widetilde{\mathbf{A}}^{l-1}(\widetilde{\mathbf{A}}^L)^{\dagger} \mathbf{Z}_b\mathbf{W}^{\dagger}\right)\mathbf{W}^{\left(1\right)}\dots\mathbf{W}^{\left(l-1\right)} \\
        & = \left(\mathbf{T}^{(l)} + \mathbf{B}^{(l)}\right)\mathbf{W}^{\left(1\right)}\dots\mathbf{W}^{\left(l-1\right)}
    \end{aligned}
\end{equation}
where $\mathbf{T}^{(l)}=\widetilde{\mathbf{A}}^{l-1}(\widetilde{\mathbf{A}}^L)^{\dagger} \mathbf{Z}_t  \mathbf{W}^{\dagger}$ and $\mathbf{B}^{(l)}=\widetilde{\mathbf{A}}^{l-1}(\widetilde{\mathbf{A}}^L)^{\dagger} \mathbf{Z}_b \mathbf{W}^{\dagger}$. Written in vector form, we naturally have the following expression
\begin{equation}
    \mathbf{h}_i^{\left(l\right)} = \left(\mathbf{t}_i^{\left(l\right)} + \mathbf{b}_i^{\left(l\right)}\right) \mathbf{W}^{\left(1\right)} \dots \mathbf{W}^{\left(l-1\right)}
\end{equation}
which completes the proof.

\section{Analysis on Linear GCN} \label{appendix:GCN_analyse}
In this section, we first prove that the centroid for each demographic group remains unchanged after message passing. Then we prove Theorem~\ref{thm:bias_amplification} (i.e., distance shrinkage). 

To complete both proofs, we present the following two propositions. The first proposition focuses on two key properties of Gilbert random graphs (Assumption~\ref{asmp:random_graph}).

\begin{proposition} \label{prop:graph_property}
    Given a Gilbert random graph $\mathcal{G}$, we have the following key properties.
    \begin{itemize}[
        align=left,
        leftmargin=2em,
        itemindent=0pt,
        labelsep=0pt,
        labelwidth=1em,
    ]
        \item \textbf{Independence between nodes.} For any node $v_i$ in $\mathcal{G}$, its node features and the node features of any of its 1-hop neighbor $v_j \in \mathcal{N}\left(v_i\right)$ are independently and identically distributed, where $\mathcal{N}\left(\cdot\right)$ represents the set of $1$-hop neighbors.
        \item \textbf{Independence between node features and topology.} The probability distribution of node features are independent to topology structure of the graph $\mathcal{G}$, i.e., the input node features $\mathbf{X}$ is independent with the adjacency matrix $\mathbf{A}$.
    \end{itemize}
\end{proposition}
\begin{proof}
    First, we first prove the \textbf{independence between nodes}. Note that in a Gilbert random graph, an edge is randomly added with a fixed probability. This process is equivalent to randomly selecting one node from the graph as the starting point and then randomly selecting another node as the endpoint. Consequently, two nodes from the same edge are obtained through simple random sampling from the given probability distribution. Therefore, node features for any node and its neighbors are independent and identically distributed.
    
    Second, since the edge generation of a Gilbert random graph is independent to node features, the probability distribution of node features are naturally independent to the topology structure.
\end{proof}

For Gilbert random graphs, since node features are independent to topology, the bias residual matrix $\mathbf{B}$ is also independent to the adjacency matrix $\mathbf{A}$. Actually, a prerequisite for this independence is that, for any node $v_i$, the information from its neighbors does not affect the mean and variance of the biased vector distribution for the node $v_i$. Therefore, we propose proposition~\ref{prop:indep_topo_feat} and give the corresponding proof.

\begin{proposition} \label{prop:indep_topo_feat}
    Given a graph $\mathcal{G}$, if bias residual matrix is independent to the adjacency matrix, then the following expressions hold true:
    
    \begin{subequations}
        \begin{align}
            \mathbb{E}_{v_i \in \mathcal{V}}\left[\mathbb{E}_{v_j \in \mathcal{N}\left(v_i\right)}\left[\mathbf{b}_j\right]\right] &= \mathbb{E}_{v_j \in \mathcal{V}}[\mathbf{b}_j] \label{subeq:mean} \\
            \mathbb{E}_{v_i \in \mathcal{V}}\left[\mathbb{E}_{v_j \in \mathcal{N}\left(v_i\right)}\left[(\mathbf{b}_j - \mathbb{E}[\mathbf{b}_j])^2\right]\right] &= \mathbb{E}_{v_j \in \mathcal{V}}[(\mathbf{b}_j - \mathbb{E}[\mathbf{b}_j])^2] \label{subeq:variance}
        \end{align}
    \end{subequations}
    where $\mathcal{N}\left(v_i\right)$ means the set of the neighbors of the node $v_i$.
\end{proposition}

\begin{proof} \label{proof:expect_equal}
    We will only prove Eq.~\eqref{subeq:mean} here, and the proof of Eq.~\eqref{subeq:variance} is exactly the same. 
    
    Actually, proving proposition~\ref{prop:indep_topo_feat} is equivalent to proving its contrapositive. Therefore, we only need to prove that given a graph $\mathcal{G}$, if the following expression hold true:
    \begin{equation} \label{eq:contra}
        \mathbb{E}_{v_i \in \mathcal{V}}\left[\mathbb{E}_{v_j \in \mathcal{N}\left(v_i\right)}\left[\mathbf{b}_j\right]\right] \neq \mathbb{E}_{v_j \in \mathcal{V}}\left[\mathbf{b}_j\right]
    \end{equation}
    then the bias residual matrix is not independent to the adjacency matrix.

    Imagine the situation that the bias residual matrix is independent to adjacency matrix. It means that we cannot obtain any information of the adjacency matrix even if we have the bias residual matrix. Then given the bias residual matrix, the probability of correctly predicting the adjacency matrix should be the same as random guessing. Considering a graph with $N$ nodes, the adjacency matrix contains a total of $N^2$ elements. With each element having two possible values, 0 or 1, the probability of random guessing correctly is given by $\textit{Pr}\left(\text{random}\right) = \frac{1}{2^{N^2}}$.

    However, we can readily identify two specific types of graph topology structure that fail to meet the requirement of Eq.~\eqref{eq:contra}. Then we can prove that if Eq.~\eqref{eq:contra} holds true, the possibility of predicting the adjacency matrix is larger than random guessing. Here are the two types of graph structures. 
    
    The first type is a graph with only self-loop edges. In such a graph, the adjacency matrix is an identity matrix. For this particular graph topology, regardless of the distribution of the bias matrix, Eq.~\eqref{subeq:mean} always holds for this particular graph topology. 

    Another type of graph is a cyclic graph. Each node in the graph is assigned a unique sequential number from $1$ to $N$, and the nodes are sorted accordingly. We denote the node with number $k$ as $v_k$. Node $v_k$ is exclusively connected to nodes $v_{k-1}$ and $v_{k+1}$. The adjacency matrix for this type of graph topology is shown below. For this graph structure, we can mathematically prove that the left-hand side of Eq.~\eqref{subeq:mean} represents the expectation obtained by summing the values twice for each node and taking the average subsequently, which is equal to the right-hand side of Eq.~\eqref{subeq:mean}.
    $$
    \begin{bmatrix}
        0 & 1 & 0 & 0 & \dots & 1 \\
        1 & 0 & 1 & 0 & \dots & 0 \\
        0 & 1 & 0 & 1 & \dots & 0 \\
        0 & 0 & 1 & 0 & \dots & 0 \\
        \vdots & \vdots & \vdots & \vdots & \ddots & \vdots \\
        1 & 0 & 0 & 0 & \dots & 0
    \end{bmatrix}
    $$

    Hence, if Eq.~\eqref{eq:contra} holds, the adjacency matrices of the aforementioned two types of graphs do not satisfy the requirement. As a result, the number of possible adjacency matrices will be less than or equal to $\left(2^{N^2} - 2\right)$. Therefore, the probability of correctly predicting the adjacency matrix, $\textit{Pr}\left(\text{correct}\right)$, satisfies $\textit{Pr}\left(\text{correct}\right) \geq \frac{1}{2^{N^2}-2} > \textit{Pr}\left(\text{random}\right) = \frac{1}{2^{N^2}}$. Thus, the contrapositive is valid because the adjacency matrix is not independent to the bias input matrix, which completes the proof of Proposition~\ref{prop:indep_topo_feat}.
\end{proof}

With the above two propositions hold, we will further prove the stable centroid after message passing (Appendix~\ref{subsec:GCN_mean_invariance}) and distance shrinkage (Theorem~\ref{thm:bias_amplification}, Appendix~\ref{subsec:GCN_dist_shrink}).





\subsection{Stable Centroid} \label{subsec:GCN_mean_invariance}
Here, we prove that, for GCNs with row normalization, the distribution centroid keeps unchanged after message passing. Mathematically, our goal is to prove the following equation.
\begin{equation}
    \mathbb{E}_{v_i \in \mathcal{V}} \left[\mathbf{b}_i^{\left(l+1\right)}\right] = \mathbb{E}_{v_j \in \mathcal{V}} \left[\mathbf{b}_j^{\left(l\right)}\right]
\end{equation}

In the $l$-th iteration of message passing, the mean of the bias residuals can be calculated as
\begin{equation}
    \begin{aligned}
        \mathbb{E}_{v_i \in \mathcal{V}} \left[\mathbf{b}_i^{\left(l+1\right)}\right]
        & = \mathbb{E}_{v_i \in \mathcal{V}} \left[\sum_{v_j \in \mathcal{N}\left(v_i\right)} \alpha_i \mathbf{b}_j^{\left(l\right)}\right]
        = \mathbb{E}_{v_i \in \mathcal{V}} \left[ \alpha_i \sum_{v_j \in \mathcal{N}\left(v_i\right)} \mathbf{b}_j^{\left(l\right)} \right] \\
        & = \mathbb{E}_{v_i \in \mathcal{V}} \left[ \alpha_i \sum_{v_j \in \mathcal{N}\left(v_i\right)} \mathbb{E}_{v_j \in \mathcal{N}\left(v_i\right)} \left[\mathbf{b}_j^{\left(l\right)} \right] \right]
        = \mathbb{E}_{v_i \in \mathcal{V}} \left[ \mathbb{E}_{v_j \in \mathcal{N}\left(v_i\right)} \left[\mathbf{b}_j^{\left(l\right)} \right] \right] \\
        & = \mathbb{E}_{v_j \in \mathcal{V}} \left[\mathbf{b}_j^{\left(l\right)}\right] \\
    \end{aligned}
\end{equation}
which completes the proof. 

\subsection{Proof of Theorem~\ref{thm:bias_amplification} (Distance Shrinkage)} \label{subsec:GCN_dist_shrink}
Let $\boldsymbol{\mu} =  \mathbf{\mu}(v_i)$ and $\hat{\mathbf{b}}_i^{\left(l+1\right)}=\mathbf{b}_i^{\left(l+1\right)} - \boldsymbol{\mu}$. We have
\begin{equation}
    \begin{aligned}
        & \mathbb{E}_{v_i}[\|\mathbf{b}^{\left(l+1\right)}_i - \mathbf{\mu}(v_i)\|_2^2] \\
        & = \mathbb{E}_{v_i \in \mathcal{V}} \left[ \left(\mathbf{b}_i^{\left(l+1\right)} - \boldsymbol{\mu}\right)^T\left(\mathbf{b}_i^{\left(l+1\right)} - \boldsymbol{\mu}\right) \right] \\
        & = \mathbb{E}_{v_i \in \mathcal{V}} \left[ \left(\hat{\mathbf{b}}_i^{\left(l+1\right)}\right)^T\hat{\mathbf{b}}_i^{\left(l+1\right)} \right] \\
        & = \mathbb{E}_{v_i \in \mathcal{V}} \left[ \left(\sum_{v_j \in \mathcal{N}\left(v_i\right)} \alpha_i \mathbf{b}_j^{\left(l\right)} - \boldsymbol{\mu}\right)^T \left(\sum_{v_j \in \mathcal{N}\left(v_i\right)} \alpha_i \mathbf{b}_j^{\left(l\right)} - \boldsymbol{\mu}\right) \right] \\
        & = \mathbb{E}_{v_i \in \mathcal{V}} \left[ \left(\sum_{v_j \in \mathcal{N}\left(v_i\right)} \alpha_i \hat{\mathbf{b}}_j^{\left(l\right)}\right)^T\left(\sum_{v_j \in \mathcal{N}\left(v_i\right)}\alpha_i \hat{\mathbf{b}}_i^{\left(l\right)}\right) \right] \\
        & = \mathbb{E}_{v_i \in \mathcal{V}} \left[ \alpha_i^2 \sum_{v_j \in \mathcal{N}\left(v_i\right)} \sum_{v_k \in \mathcal{N}\left(v_i\right)} \left(\hat{\mathbf{b}}_j^{\left(l\right)}\right)^T\hat{\mathbf{b}}_k^{\left(l\right)}\right] \\
        & = \underbrace{\mathbb{E}_{v_i \in \mathcal{V}} \left[ \alpha_i^2 \sum_{v_j \in \mathcal{N}\left(v_i\right)}  \left(\hat{\mathbf{b}}_j^{\left(l\right)}\right)^T\hat{\mathbf{b}}_j^{\left(l\right)}\right]}_{\text{\ding{172}}} + \underbrace{\mathbb{E}_{v_i \in \mathcal{V}} \left[ \alpha_i^2 \sum_{v_j \in \mathcal{N}\left(v_i\right)} \sum_{v_k \in \mathcal{N}\left(v_i\right) \setminus \left\{v_j\right\}}\left(\hat{\mathbf{b}}_j^{\left(l\right)}\right)^T\hat{\mathbf{b}}_k^{\left(l\right)} \right]}_{\text{\ding{173}}}
    \end{aligned}
\end{equation}

For \ding{172}, we have
\begin{equation}
    \begin{aligned}
        &\mathbb{E}_{v_i \in \mathcal{V}} \left[ \alpha_i^2 \sum_{v_j \in \mathcal{N}\left(v_i\right)}  \left(\hat{\mathbf{b}}_j^{\left(l\right)}\right)^T\hat{\mathbf{b}}_j^{\left(l\right)}\right] \\
        &= \mathbb{E}_{v_i \in \mathcal{V}} \left[ \alpha_i^2 \sum_{v_j \in \mathcal{N}\left(v_i\right)} \mathbb{E}_{v_j \in \mathcal{N}\left(v_i\right)} \left[ \left(\hat{\mathbf{b}}_j^{\left(l\right)}\right)^T\hat{\mathbf{b}}_j^{\left(l\right)} \right]\right] \\
        &= \mathbb{E}_{v_i \in \mathcal{V}} \left[ \alpha_i \mathbb{E}_{v_j \in \mathcal{N}\left(v_i\right)} \left[ \left(\hat{\mathbf{b}}_j^{\left(l\right)}\right)^T\hat{\mathbf{b}}_j^{\left(l\right)} \right]\right] \\
        &= \mathbb{E}_{v_i \in \mathcal{V}} \left[ \alpha_i \right] \mathbb{E}_{v_i \in \mathcal{V}} \left[ \mathbb{E}_{v_j \in \mathcal{N}\left(v_i\right)} \left[ \left(\hat{\mathbf{b}}_j^{\left(l\right)}\right)^T\hat{\mathbf{b}}_j^{\left(l\right)} \right]\right] \\
        &= \mathbb{E}_{v_i \in \mathcal{V}} \left[ \frac{1}{d_i} \right] \mathbb{E}_{v_i \in \mathcal{V}} \left[ \left(\hat{\mathbf{b}}_i^{\left(l\right)}\right)^T\hat{\mathbf{b}}_i^{\left(l\right)} \right]
    \end{aligned}
\end{equation}
where $d_i$ is the degree of the $i$-th node. For \ding{173}, we have
\begin{equation}
    \begin{aligned}
        & \mathbb{E}_{v_i \in \mathcal{V}} \left[ \alpha_i^2 \sum_{v_j \in \mathcal{N}\left(v_i\right)} \sum_{v_k \in \mathcal{N}\left(v_i\right) \setminus \left\{j\right\}} \left(\hat{\mathbf{b}}_j^{\left(l\right)}\right)^T\hat{\mathbf{b}}_k^{\left(l\right)} \right] \\
        & = \mathbb{E}_{v_i \in \mathcal{V}} \left[ \alpha_i^2 \right] \mathbb{E}_{v_i \in \mathcal{V}} \left[\sum_{v_j \in \mathcal{N}\left(v_i\right)} \sum_{v_k \in \mathcal{N}\left(v_i\right) \setminus \left\{j\right\}} \left(\hat{\mathbf{b}}_j^{\left(l\right)}\right)^T\hat{\mathbf{b}}_k^{\left(l\right)} \right] \\
        & = \mathbb{E}_{v_i \in \mathcal{V}} \left[ \alpha_i^2 \right] \frac{1}{\vert \mathcal{V} \vert} \sum_{v_i \in \mathcal{V}}\sum_{v_j \in \mathcal{N}\left(v_i\right)} \sum_{v_k \in \mathcal{N}\left(v_i\right) \setminus \left\{j\right\}} \left(\hat{\mathbf{b}}_j^{\left(l\right)}\right)^T\hat{\mathbf{b}}_k^{\left(l\right)} \\
        & = \mathbb{E}_{v_i \in \mathcal{V}} \left[ \alpha_i^2 \right] \frac{1}{\vert \mathcal{V} \vert} \sum_{v_j \in \mathcal{V}} \sum_{v_k \in \mathcal{N}^2\left(v_j\right) \setminus \left\{j\right\}} \left(\hat{\mathbf{b}}_j^{\left(l\right)}\right)^T\hat{\mathbf{b}}_k^{\left(l\right)}
    \end{aligned}
\end{equation}
where $\mathcal{N}^2\left(v_j\right)$ is the set of the $2$-hop neighbors of the node $v_j$. Since the bias residual matrix $\mathbf{B}$ is independent to the adjacency matrix $\mathbf{A}$, it is also independent to $\widetilde{\mathbf{A}} = \phi(\mathbf{A}^2)$, where $\phi(\cdot)$ is the function that sets the diagonal elements of a given matrix to 0. Let $\widetilde{\mathcal{N}}\left(v\right)$ represent the neighborhood of node $v$ in the new adjacency matrix $\widetilde{\mathbf{A}}$. Please note that the new neighborhood  $\widetilde{\mathcal{N}}\left(\cdot\right)$ is the 2-hop neighborhood excluding the node itself, \ie, $\widetilde{\mathcal{N}}\left(v_i\right) = \mathcal{N}^2\left(v_i\right) \setminus \left\{i\right\}$. Then we have

\begin{equation}
    \begin{aligned}
        &\mathbb{E}_{v_i \in \mathcal{V}} \left[ \alpha_i^2 \sum_{v_j \in \mathcal{N}\left(v_i\right)} \sum_{v_k \in \mathcal{N}\left(v_i\right) \setminus \left\{j\right\}} \left(\hat{\mathbf{b}}_j^{\left(l\right)}\right)^T\hat{\mathbf{b}}_k^{\left(l\right)} \right] \\
        &= \mathbb{E}_{v_i \in \mathcal{V}} \left[ \alpha_i^2 \right] \frac{1}{\lvert \mathcal{V} \rvert} \sum_{v_j \in \mathcal{V}} \sum_{v_k \in \mathcal{N}^2\left(v_j\right) \setminus \left\{j\right\}} \left(\hat{\mathbf{b}}_j^{\left(l\right)}\right)^T\hat{\mathbf{b}}_k^{\left(l\right)} \\
        &= \mathbb{E}_{v_i \in \mathcal{V}} \left[ \alpha_i^2 \right] \mathbb{E}_{v_j \in \mathcal{V}} \left[ \sum_{v_k \in \widetilde{\mathcal{N}}\left(v_j\right)} \left(\hat{\mathbf{b}}_j^{\left(l\right)}\right)^T\hat{\mathbf{b}}_k^{\left(l\right)} \right] \\
        &= \mathbb{E}_{v_i \in \mathcal{V}} \left[ \alpha_i^2 \right] \mathbb{E}_{v_j \in \mathcal{V}} \left[ \left(\hat{\mathbf{b}}_j^{\left(l\right)}\right)^T \sum_{v_k \in \widetilde{\mathcal{N}}\left(v_j\right)} \hat{\mathbf{b}}_k^{\left(l\right)} \right] \\
    \end{aligned}
\end{equation}
Since there is no self-loop in $\widetilde{\mathbf{A}}$, the bias residual $\hat{\mathbf{b}}_j$ is different from the bias residual vector $\hat{\mathbf{b}}_k$, hence independent to $\hat{\mathbf{b}}_k$. Therefore, we have

\begin{equation}
    \begin{aligned}
        &\mathbb{E}_{v_i \in \mathcal{V}} \left[ \alpha_i^2 \sum_{v_j \in \mathcal{N}\left(v_i\right)} \sum_{v_k \in \mathcal{N}\left(v_i\right) \setminus \left\{j\right\}} \left(\hat{\mathbf{b}}_j^{\left(l\right)}\right)^T\hat{\mathbf{b}}_k^{\left(l\right)} \right] \quad \text{\ding{173}}\\
        &= \mathbb{E}_{v_i \in \mathcal{V}} \left[ \alpha_i^2 \right] \mathbb{E}_{v_j \in \mathcal{V}} \left[ \left(\hat{\mathbf{b}}_j^{\left(l\right)}\right)^T \sum_{v_k \in \widetilde{\mathcal{N}}\left(v_j\right)} \hat{\mathbf{b}}_k^{\left(l\right)} \right] \\
        &= \mathbb{E}_{v_i \in \mathcal{V}} \left[ \alpha_i^2 \right] \mathbb{E}_{v_j \in \mathcal{V}} \left[ \left(\hat{\mathbf{b}}_j^{\left(l\right)}\right)^T  \right] \mathbb{E}_{v_j \in \mathcal{V}} \left[\sum_{v_k \in \widetilde{\mathcal{N}}\left(v_j\right)} \hat{\mathbf{b}}_k^{\left(l\right)} \right] \\
        &= \mathbb{E}_{v_i \in \mathcal{V}} \left[ \alpha_i^2 \right] \mathbf{0}^T \mathbb{E}_{v_j \in \mathcal{V}} \left[\sum_{v_k \in \widetilde{\mathcal{N}}\left(v_j\right)} \hat{\mathbf{b}}_k^{\left(l\right)} \right] \\
        &= 0
    \end{aligned}
\end{equation}

Combining everything together, we have the following equation holds.
\begin{equation} \label{eq:dist_shrink_GCN}
    \mathbb{E}_{v_i \in \mathcal{V}} \left[ \left(\hat{\mathbf{b}}_i^{\left(l+1\right)}\right)^T \hat{\mathbf{b}}_i^{\left(l+1\right)} \right] = \mathbb{E}_{v_i \in \mathcal{V}} \left[ \frac{1}{d_i} \right] \mathbb{E}_{v_i \in \mathcal{V}} \left[ \left(\hat{\mathbf{b}}_i^{\left(l\right)}\right)^T\hat{\mathbf{b}}_i^{\left(l\right)} \right]
\end{equation}

Based on Eq.~\eqref{eq:dist_shrink_GCN}, for every demographic group, the corresponding distance will shrink at the rate proportional to the reciprocal of node degree, which completes the proof of Theorem~\ref{thm:bias_amplification}.

\section{Analysis on BeMap} \label{appendix:BeMap_analyse}
In this section, we prove \textbf{centroid consistency} (Appendix~\ref{subsec:centroid_consist_append}) and \textbf{distance shrinkage} (Appendix~\ref{subsec:bemap_dist_shrink}). We first present a variant of Proposition~\ref{prop:indep_topo_feat} in Proposition~\ref{prop:indep_topo_feat_multigroup}.

\begin{proposition} \label{prop:indep_topo_feat_multigroup}
    Given a graph $\mathcal{G}$ and a specific demographic group $\mathcal{S}$, if bias residuals from the demographic group is independent to the adjacency matrix, then the following equations holds
    \begin{equation} \label{subeq:mean_group}
        \mathbb{E}_{v_i \in \mathcal{V}}\left[\mathbb{E}_{v_j \in \mathcal{N}\left(v_i\right)\cap \mathcal{S}}\left[\mathbf{b}_j\right]\right] = \mathbb{E}_{v_j \in \mathcal{S}}\left[\mathbf{b}_j\right]
    \end{equation}
    \vspace{-4mm}
    \begin{equation} \label{subeq:variance_group}
        \begin{aligned}
            &\mathbb{E}_{v_i \in \mathcal{V}}\left[\mathbb{E}_{v_j \in \mathcal{N}\left(v_i\right)\cap \mathcal{S}}\left[\left(\mathbf{b}_j - \mathbb{E}_{v_j\in \mathcal{S}}\left[\mathbf{b}_j\right]\right)^2\right]\right] = \mathbb{E}_{v_j \in \mathcal{S}}\left[\left(\mathbf{b}_j - \mathbb{E}_{v_j \in \mathcal{S}}\left[\mathbf{b}_j\right]\right)^2\right]
        \end{aligned}
    \end{equation}
    where $\mathcal{N}\left(v_i\right)$ means the set of the neighbors of the node $v_i$. 
\end{proposition}

\begin{proof} \label{proof:expect_equal_group}
    Similar to the proof of Proposition~\ref{prop:indep_topo_feat}, we prove the contrapositive related to Eq.~\eqref{subeq:mean_group}. Specifically, we want to prove that given a graph and a specific demographic group $\mathcal{S}$, if the following equation holds
    \begin{equation} \label{eq:contra_mean_group}
        \mathbb{E}_{v_i \in \mathcal{V}}\left[\mathbb{E}_{v_j \in \mathcal{N}\left(v_i\right)\cap \mathcal{S}}\left[\mathbf{b}_j\right]\right] \neq \mathbb{E}_{v_j \in \mathcal{S}}\left[\mathbf{b}_j\right]
    \end{equation}
    the bias residual matrix from the demographic group $\mathcal{S}$ is not independent to the adjacency matrix.
    To begin with, same as Appendix~\ref{proof:expect_equal}, when the bias residuals are independent to the adjacency matrix, then we can easily have the probability of correctly predicting the adjacency matrix $Pr(random)=\frac{1}{2^{N^2}}$ with $N$ representing the number of nodes.

    Then we will prove that if Eq.~\eqref{eq:contra_mean_group} is true, the probability of correctly predicting the adjacency matrix will be lager than that of random guessing. First, we give every node a unique index. For the nodes from the demographic group $\mathcal{S}$, we assign them consecutive numbers from 1 to $\lvert \mathcal{S} \rvert$, and assign the nodes that do not belong to the demographic group $\mathcal{S}$ the consecutive numbers from $\lvert \mathcal{S} \rvert + 1$ to $N$. Then we sort all the nodes, and use $v_k$ to represent the node with the index $k$.

    Consider the given adjacency matrix $\mathbf{A}^*$ listed below. Specifically, the matrix $\mathbf{M}_1$ is a $\lvert \mathcal{S} \rvert \times \lvert \mathcal{S} \rvert$ matrix which can be any type of adjacency matrices described in Proof~\ref{proof:expect_equal}, i.e., the identity matrix or the adjacency matrix of a cyclic graph. The matrix $\mathbf{M}_2$ is a $(N - \lvert \mathcal{S} \rvert) \times \lvert \mathcal{S} \rvert$ matrix with all the elements being $1$. The matrix $\mathbf{M}_3$ can be any matrix whose shape is $(N - \lvert \mathcal{S} \rvert) \times (N - \lvert \mathcal{S} \rvert)$. We can naturally prove that the adjacency matrix $\mathbf{A}^*$ does not meet the requirement of Eq.~\eqref{eq:contra_mean_group}. For any node $v_k$ from the demographic group $\mathcal{S}$, the neighbor relationship between the node $v_i$ and its neighbors $\mathcal{N}\left(v_i\right) \cap \mathcal{S}$ from the demographic group $\mathcal{S}$ can be depicted by the matrix $\mathbf{M}_1$. Then due to Proof~\ref{proof:expect_equal}, the expression of $\mathbb{E}_{v_i \in \mathcal{S}}\left[\mathbb{E}_{v_j \in \mathcal{N}\left(v_i\right)\cap \mathcal{S}}\left[\mathbf{b}_j\right]\right] = \mathbb{E}_{v_j \in \mathcal{S}}\left[\mathbf{b}_j\right]$ holds true for the nodes from the demographic group $\mathcal{S}$. For the nodes $v_i$ which are not from the demographic group $\mathcal{S}$, their neighbors $\mathcal{N}\left(v_i\right) \cap \mathcal{S}$ from the demographic group $\mathcal{S}$ are the demographic group $\mathcal{S}$ since any element in the matrix $\mathbf{M}_2$ is $1$. Then naturally, the expression of $\mathbb{E}_{v_i \in \mathcal{V} -\mathcal{S}}\left[\mathbb{E}_{v_j \in \mathcal{N}\left(v_i\right)\cap \mathcal{S}}\left[\mathbf{b}_j\right]\right] = \mathbb{E}_{v_j \in \mathcal{S}}\left[\mathbf{b}_j\right]$ holds true. Therefore, the given adjacency matrix $\mathbf{A}^*$ cannot satisfy Eq.~\eqref{eq:contra_mean_group}. 
    $$
    \mathbf{A}^*=
    \begin{bmatrix}
        \mathbf{M}_1 & \mathbf{M}_2^T \\
        \mathbf{M}_2 & \mathbf{M}_3 \\
    \end{bmatrix}
    $$

    Since the matrix $\mathbf{M}_1$ has at least two possible solutions listed in Proof~\ref{proof:expect_equal} and the matrix $\mathbf{M}_3$ has $2^{(N - \lvert \mathcal{S} \rvert) \times (N - \lvert \mathcal{S} \rvert)}$ possible solutions, the matrix $\mathbf{A}^*$ has at least $2^{\left(N - \lvert \mathcal{S} \rvert\right) \times \left(N - \lvert \mathcal{S} \rvert\right) + 1}$ possible solutions. Therefore, the probability of correcting predicting the adjacency matrix satisfies $\textit{Pr}\left(\text{correct}\right) \geq \frac{1}{2^{N^2} - 2^{\left(N - \lvert \mathcal{S} \rvert\right) \times \left(N - \lvert \mathcal{S} \rvert\right) + 1}} > \frac{1}{2^{N^2}} = \textit{Pr}\left(\text{random}\right)$. Thus, the contrapositive is valid, which completes the proof.
\end{proof}

\subsection{Centroid Consistency} \label{subsec:centroid_consist_append}

First, in the following paper, the last subscript separated by the comma on the lower right corner are used to refer to the demographic group to which bias vectors belongs. For example, $\mathbf{b}_{i,s}$ represents the bias vector of the node $v_i$ with the sensitive attribute of $s$. Then the mean of the biased vector $\mathbf{b}_i^{\left(l+1\right)}$ can be calculated as

\begin{equation} \label{eq:centroid_consist_BeMap_app}
    \begin{aligned}
        \mathbb{E}_{v_i \in \mathcal{V}} \left[  \mathbf{b}_i^{\left(l+1\right)}\right]
        &= \mathbb{E}_{v_i \in \mathcal{V}} \left[ \sum_{v_j \in \mathcal{N}_0\left(v_i\right)} \alpha_i \mathbf{b}_{j, 0}^{\left(l\right)} + \sum_{v_j \in \mathcal{N}_1\left(v_i\right)} \alpha_i \mathbf{b}_{j, 1}^{\left(l\right)} \right] \\
         &= \mathbb{E}_{v_i \in \mathcal{V}} \left[  \sum_{v_j \in \mathcal{N}_0\left(v_i\right)} \alpha_i \mathbf{b}_{j, 0}^{\left(l\right)} \right] + \mathbb{E}_{v_i \in \mathcal{V}} \left[ \sum_{v_j \in \mathcal{N}_1\left(v_i\right)} \alpha_i \mathbf{b}_{j, 1}^{\left(l\right)} \right] \\
        &= \mathbb{E}_{v_i \in \mathcal{V}} \left[ \alpha_i \sum_{v_j \in \mathcal{N}_0\left(v_i\right)} \mathbf{b}_{j, 0}^{\left(l\right)} \right] + \mathbb{E}_{v_i \in \mathcal{V}} \left[ \alpha_i \sum_{v_j \in \mathcal{N}_1\left(v_i\right)} \mathbf{b}_{j, 1}^{\left(l\right)} \right] \\
        &= \mathbb{E}_{v_i \in \mathcal{V}} \left[ \alpha_i \vert \mathcal{N}_0\left(v_i\right) \vert \mathbb{E}_{v_j \in \mathcal{N}_0\left(v_i\right)} \left[ \mathbf{b}_{j, 0}^{\left(l\right)} \right] \right] + \mathbb{E}_{v_i \in \mathcal{V}} \left[ \alpha_i \vert \mathcal{N}_1\left(v_i\right) \vert \mathbb{E}_{v_j \in \mathcal{N}_1 \left(v_i\right)} \left[\mathbf{b}_{j, 1}^{\left(l\right)} \right] \right] \\
        &= \mathbb{E}_{v_i \in \mathcal{V}} \left[ r_0 \mathbb{E}_{v_j \in \mathcal{N}_0\left(v_i\right)} \left[ \mathbf{b}_{j, 0}^{\left(l\right)} \right] \right] + \mathbb{E}_{v_i \in \mathcal{V}} \left[ r_1 \mathbb{E}_{v_j \in \mathcal{N}_1 \left(v_i\right)} \left[\mathbf{b}_{j, 1}^{\left(l\right)} \right] \right] \\
        &= r_0 \mathbb{E}_{v_j \in \mathcal{V}_0} \left[ \mathbf{b}_{j, 0}^{\left(l\right)} \right] + r_1 \mathbb{E}_{v_j \in \mathcal{V}_1} \left[ \mathbf{b}_{j, 1}^{\left(l\right)} \right]
    \end{aligned}
\end{equation}
where $\mathcal{N}_s\left(v_i\right)$ represent the neighbors of the node $v_i$ with the sensitive attribute $s$.

For non-binary sensitive attribute $s \in \left\{0, \dots, S-1\right\}$, we can easily rewrite the above results as
\begin{equation} \label{eq:centroid_consist_BeMap_multisens_app}
    \mathbb{E}_{v_i \in \mathcal{V}} \left[  \mathbf{b}_i^{\left(l+1\right)}\right] = \sum_{s=0}^{S-1} r_s \mathbb{E}_{v_j \in \mathcal{V}_s} \left[ \mathbf{b}_{j, s}^{\left(l\right)} \right]
\end{equation}

Actually, Eq.~\eqref{eq:centroid_consist_BeMap_app} and Eq.~\eqref{eq:centroid_consist_BeMap_multisens_app} hold for any node regardless of its sensitive attribute. Therefore, we successfully prove \textbf{centroid consistency}.
 
\subsection{Distance Shrinkage} \label{subsec:bemap_dist_shrink}
We will calculate the distance $\mathbb{E}_{v_i}[\|\mathbf{b}_i^{(l+1)} - \mathbf{\bar \mu}\|_2^2] = \mathbb{E}_{v_i \in \mathcal{V}} \left[ \left(\hat{\mathbf{b}}_i^{\left(l+1\right)}\right)^T\hat{\mathbf{b}}_i^{\left(l+1\right)} \right]$ here. First, we simplify the expression of the bias residual $\hat{\mathbf{b}}_i^{\left(l+1\right)}$

\begin{equation}
    \begin{aligned}
            \hat{\mathbf{b}}_i^{\left(l+1\right)}
            &= \mathbf{b}_i^{\left(l+1\right)} -\mathbb{E}_{v_i \in \mathcal{V}}\left[ \mathbf{b}_i^{\left(l+1\right)} \right] \\
            &= \sum_{v_j \in \mathcal{N}_0\left(v_i\right)} \alpha_i \mathbf{b}_{j, 0}^{\left(l\right)} + \sum_{v_j \in \mathcal{N}_1\left(v_i\right)} \alpha_i \mathbf{b}_{j, 1}^{\left(l\right)} - \mathbb{E}_{v_i \in \mathcal{V}}\left[\sum_{v_j \in \mathcal{N}_0\left(v_i\right)} \alpha_i \mathbf{b}_{j, 0}^{\left(l\right)}\right] - \mathbb{E}_{v_i \in \mathcal{V}}\left[\sum_{v_j \in \mathcal{N}_1\left(v_i\right)} \alpha_i \mathbf{b}_{j, 1}^{\left(l\right)} \right] \\
            &= \sum_{v_j \in \mathcal{N}_0\left(v_i\right)} \alpha_i \mathbf{b}_{j, 0}^{\left(l\right)} + \sum_{v_j \in \mathcal{N}_1\left(v_i\right)} \alpha_i \mathbf{b}_{j, 1}^{\left(l\right)} - \mathbb{E}_{v_i \in \mathcal{V}} \left[ r_0 \mathbb{E}_{v_j \in \mathcal{N}_0\left(v_i\right)} \left[\mathbf{b}_{j, 0}^{\left(l\right)} \right] \right] + \mathbb{E}_{v_i \in \mathcal{V}}\left[r_1 \mathbb{E}_{v_j \in \mathcal{N}_1\left(v_i\right)} \left[\mathbf{b}_{j, 1}^{\left(l\right)} \right] \right] \\
            &= \sum_{v_j \in \mathcal{N}_0\left(v_i\right)} \alpha_i \mathbf{b}_{j, 0}^{\left(l\right)} + \sum_{v_j \in \mathcal{N}_1\left(v_i\right)} \alpha_i \mathbf{b}_{j, 1}^{\left(l\right)} - r_0 \mathbb{E}_{v_j \in \mathcal{V}_0}\left[ \mathbf{b}_{j, 0}^{\left(l\right)}\right] - r_1 \mathbb{E}_{v_j \in \mathcal{V}_1}\left[ \mathbf{b}_{j, 1}^{\left(l\right)}\right] \\
            &= \sum_{v_j \in \mathcal{N}_0\left(v_i\right)} \alpha_i \left(\mathbf{b}_{j, 0}^{\left(l\right)} - \mathbb{E}_{v_j \in \mathcal{V}_0}\left[ \mathbf{b}_{j, 0}^{\left(l\right)}\right]\right) + \sum_{v_j \in \mathcal{N}_1\left(v_i\right)} \alpha_i \left( \mathbf{b}_{j, 1}^{\left(l\right)} - \mathbb{E}_{v_j \in \mathcal{V}_1}\left[ \mathbf{b}_{j, 1}^{\left(l\right)}\right]\right) \\
            &= \sum_{v_j \in \mathcal{N}_0\left(v_i\right)} \alpha_i \hat{\mathbf{b}}_{j, 0}^{\left(l\right)} + \sum_{v_j \in \mathcal{N}_1\left(v_i\right)} \alpha_i \hat{\mathbf{b}}_{j, 1}^{\left(l\right)}
    \end{aligned}
\end{equation}
where $\mathcal{N}_s\left(v_i\right)$ represent the set of neighbors with the sensitive attribute of $s$ for the node $v_i$, $\mathcal{V}_s$ represent the set of all the nodes with the sensitive attribute of $s$, and $\hat{\mathbf{b}}_{j, s}^{\left(l\right)} = \mathbf{b}_{j, s}^{\left(l\right)} - \mathbb{E}_{v_j \in \mathcal{V}_s}\left[ \mathbf{b}_{j, s}^{\left(l\right)}\right], s \in \left\{0, 1\right\}$. Obviously, the mean of the distribution of $\hat{\mathbf{b}}_{j, s}^{\left(l\right)}$ is $\mathbf{0}$. Then the expectation of distance will be expressed as:

\begin{equation}
    \begin{aligned}
        &\mathbb{E}_{v_i \in \mathcal{V}} \left[ \left(\hat{\mathbf{b}}_i^{\left(l+1\right)}\right)^T\hat{\mathbf{b}}_i^{\left(l+1\right)} \right] \\
        &=\mathbb{E}_{v_i \in \mathcal{V}} \left[ \left(\sum_{v_j \in \mathcal{N}_0\left(v_i\right)} \alpha_i \hat{\mathbf{b}}_{j, 0}^{\left(l\right)} + \sum_{v_j \in \mathcal{N}_1\left(v_i\right)} \alpha_i \hat{\mathbf{b}}_{j, 1}^{\left(l\right)}\right)^T \left(\sum_{v_j \in \mathcal{N}_0\left(v_i\right)} \alpha_i \hat{\mathbf{b}}_{j, 0}^{\left(l\right)} + \sum_{v_j \in \mathcal{N}_1\left(v_i\right)} \alpha_i \hat{\mathbf{b}}_{j, 1}^{\left(l\right)}\right)\right] \\
        &= \mathbb{E}_{v_i \in \mathcal{V}}\left[ \alpha_i^2 \sum_{v_j \in \mathcal{N}_0\left(v_i\right)} \sum_{v_k \in \mathcal{N}_0\left(v_i\right)} \left(\hat{\mathbf{b}}_{j, 0}^{\left(l\right)}\right)^T \hat{\mathbf{b}}_{k, 0}^{\left(l\right)} \right] + \mathbb{E}_{v_i \in \mathcal{V}}\left[ \alpha_i^2 \sum_{v_j \in \mathcal{N}_1\left(v_i\right)} \sum_{v_k \in \mathcal{N}_1\left(v_i\right)} \left(\hat{\mathbf{b}}_{j, 1}^{\left(l\right)}\right)^T \hat{\mathbf{b}}_{k, 1}^{\left(l\right)} \right] \\
        & \qquad + 2 \mathbb{E}_{v_i \in \mathcal{V}} \left[ \alpha_i^2 \sum_{v_j \in \mathcal{N}_0\left(v_i\right)} \sum_{v_k \in \mathcal{N}_1\left(v_i\right)} \left(\hat{\mathbf{b}}_{j, 0}^{\left(l\right)}\right)^T \hat{\mathbf{b}}_{k, 1}^{\left(l\right)} \right] \\
        &\leq 2 \underbrace{\mathbb{E}_{v_i \in \mathcal{V}}\left[ \alpha_i^2 \sum_{v_j \in \mathcal{N}_0\left(v_i\right)} \sum_{v_k \in \mathcal{N}_0\left(v_i\right)} \left(\hat{\mathbf{b}}_{j, 0}^{\left(l\right)}\right)^T \hat{\mathbf{b}}_{k, 0}^{\left(l\right)} \right]}_{\text{\ding{172}}} + 2 \underbrace{\mathbb{E}_{v_i \in \mathcal{V}}\left[ \alpha_i^2 \sum_{v_j \in \mathcal{N}_1\left(v_i\right)} \sum_{v_k \in \mathcal{N}_1\left(v_i\right)} \left(\hat{\mathbf{b}}_{j, 1}^{\left(l\right)}\right)^T \hat{\mathbf{b}}_{k, 1}^{\left(l\right)} \right]}_{\text{\ding{173}}} \\
    \end{aligned}
\end{equation}

The first item \ding{172} can be simplified as
\begin{equation}
    \begin{aligned}
        &\mathbb{E}_{v_i \in \mathcal{V}}\left[ \alpha_i^2 \sum_{v_j \in \mathcal{N}_0\left(v_i\right)} \sum_{v_k \in \mathcal{N}_0\left(v_i\right)} \left(\hat{\mathbf{b}}_{j, 0}^{\left(l\right)}\right)^T \hat{\mathbf{b}}_{k, 0}^{\left(l\right)} \right] \\
        &= \mathbb{E}_{v_i \in \mathcal{V}} \left[ \alpha_i^2 \sum_{v_j \in \mathcal{N}_0\left(v_i\right)} \left(\hat{\mathbf{b}}_{j, 0}^{\left(l\right)}\right)^T \hat{\mathbf{b}}_{j, 0}^{\left(l\right)}\right] + \mathbb{E}_{v_i \in \mathcal{V}} \left[ \alpha_i^2 \sum_{v_j \in \mathcal{N}_0\left(v_i\right)} \sum_{v_k \in \mathcal{N}_0\left(v_i\right) \setminus \left\{j\right\}}\left(\hat{\mathbf{b}}_{j, 0}^{\left(l\right)}\right)^T \hat{\mathbf{b}}_{k, 0}^{\left(l\right)}\right] \\
        &= \mathbb{E}_{v_i \in \mathcal{V}} \left[ \alpha_i^2 d^{(0)}_i \mathbb{E}_{v_j \in \mathcal{N}_0\left(v_i\right)}\left[ \left(\hat{\mathbf{b}}_{j, 0}^{\left(l\right)}\right)^T \hat{\mathbf{b}}_{j, 0}^{\left(l\right)}\right] \right] + \mathbb{E}_{v_i \in \mathcal{V}} \left[\alpha_i^2\right] \frac{\lvert \mathcal{V}_0 \rvert}{\lvert \mathcal{V} \rvert} \mathbb{E}_{v_j \in \mathcal{V}_0} \left[ \left(\hat{\mathbf{b}}_{j, 0}^{\left(l\right)}\right)^T \sum_{v_k \in \widetilde{\mathcal{N}}_0\left(v_j\right)} \hat{\mathbf{b}}_{k, 0}^{\left(l\right)} \right] \\
        &= \mathbb{E}_{v_i \in \mathcal{V}} \left[ \alpha_i^2 d^{(0)}_i \right] \mathbb{E}_{v_j \in \mathcal{V}_0}\left[ \left(\hat{\mathbf{b}}_{j, 0}^{\left(l\right)}\right)^T \hat{\mathbf{b}}_{j, 0}^{\left(l\right)}\right] + \mathbb{E}_{v_i \in \mathcal{V}} \left[\alpha_i^2\right] \frac{\lvert \mathcal{V}_0 \rvert}{\lvert \mathcal{V} \rvert} \mathbb{E}_{v_j \in \mathcal{V}_0} \left[ \left(\hat{\mathbf{b}}_{j, 0}^{\left(l\right)}\right)^T\right] \mathbb{E}_{v_j \in \mathcal{V}_0}\left[ \sum_{v_k \in \widetilde{\mathcal{N}}_0\left(v_j\right)} \hat{\mathbf{b}}_{k, 0}^{\left(l\right)}\right] \\
        &= \mathbb{E}_{v_i \in \mathcal{V}} \left[ r_0 \alpha_i\right] \mathbb{E}_{v_j \in \mathcal{V}_0}\left[ \left(\hat{\mathbf{b}}_{j, 0}^{\left(l\right)}\right)^T \hat{\mathbf{b}}_{j, 0}^{\left(l\right)}\right] + 0\\
        &=r_0\mathbb{E}_{v_i \in \mathcal{V}} \left[ \frac{1}{d_i}\right] \mathbb{E}_{v_j \in \mathcal{V}_0}\left[ \left(\hat{\mathbf{b}}_{j, 0}^{\left(l\right)}\right)^T \hat{\mathbf{b}}_{j, 0}^{\left(l\right)}\right]
    \end{aligned}
\end{equation}
where $d_i^{(s)}$ the the number of the node $v_i$'s neighbors with the sensitive attribute of $s$, and $\widetilde{\mathcal{N}}_s\left(v_i\right)$ is the neighbors from $\widetilde{\mathcal{N}}\left(v_i\right)$ with the sensitive attribute of $s, s\in \left\{0, 1\right\}$.

Similarly, the second item \ding{173} can be simplified as:
\begin{equation}
    \mathbb{E}_{v_i \in \mathcal{V}}\left[ \alpha_i^2 \sum_{v_j \in \mathcal{N}_1\left(v_i\right)} \sum_{v_k \in \mathcal{N}_1\left(v_i\right)} \left(\hat{\mathbf{b}}_{j, 1}^{\left(l\right)}\right)^T \hat{\mathbf{b}}_{k, 1}^{\left(l\right)} \right] = r_1\mathbb{E}_{v_i \in \mathcal{V}} \left[ \frac{1}{d_i}\right] \mathbb{E}_{v_j \in \mathcal{V}_1}\left[ \left(\hat{\mathbf{b}}_{j, 1}^{\left(l\right)}\right)^T \hat{\mathbf{b}}_{j, 1}^{\left(l\right)}\right]
\end{equation}

Combining everything together, we have the following relationship between the distances to centroids before and after message passing.

\begin{equation}
    \begin{aligned}
        & \mathbb{E}_{v_i \in \mathcal{V}} \left[ \left(\hat{\mathbf{b}}_i^{\left(l+1\right)}\right)^T\hat{\mathbf{b}}_i^{\left(l+1\right)} \right] \\
        & = r_0\mathbb{E}_{v_i \in \mathcal{V}} \left[ \frac{2}{d_i}\right] \mathbb{E}_{v_j \in \mathcal{V}_0}\left[ \left(\hat{\mathbf{b}}_{j, 0}^{\left(l\right)}\right)^T \hat{\mathbf{b}}_{j, 0}^{\left(l\right)}\right] + r_1\mathbb{E}_{v_i \in \mathcal{V}} \left[ \frac{2}{d_i}\right] \mathbb{E}_{v_j \in \mathcal{V}_1}\left[ \left(\hat{\mathbf{b}}_{j, 1}^{\left(l\right)}\right)^T \hat{\mathbf{b}}_{j, 1}^{\left(l\right)}\right] \\
        &= \mathbb{E}_{v_i \in \mathcal{V}} \left[ \frac{2}{d_i}\right] \left( r_0 \mathbb{E}_{v_j \in \mathcal{V}_0}\left[ \left(\hat{\mathbf{b}}_{j, 0}^{\left(l\right)}\right)^T \hat{\mathbf{b}}_{j, 0}^{\left(l\right)}\right] + r_1 \mathbb{E}_{v_j \in \mathcal{V}_1}\left[ \left(\hat{\mathbf{b}}_{j, 1}^{\left(l\right)}\right)^T \hat{\mathbf{b}}_{j, 1}^{\left(l\right)}\right]\right)
    \end{aligned}
\end{equation}

For the non-binary sensitive attribute $s \in \left\{0, 1, \dots, S-1\right\}$, similarly, we have 
\begin{equation}
    \mathbb{E}_{v_i \in \mathcal{V}} \left[ \left(\hat{\mathbf{b}}_i^{\left(l+1\right)}\right)^T\hat{\mathbf{b}}_i^{\left(l+1\right)} \right] = \mathbb{E}_{v_i \in \mathcal{V}} \left[\frac{S}{d_i}\right] \sum_{s=0}^{S-1} r_s \mathbb{E}_{v_j \in \mathcal{V}_s}\left[ (\hat{\mathbf{b}}_{j, s}^{\left(l\right)})^T \hat{\mathbf{b}}_{j, s}^{\left(l\right)}\right]
\end{equation}

Let us revisit the discussion concerning binary sensitive attribute. Since the local neighborhood is large enough mentioned in Theorem~\ref{thm:bemap}, i.e., $d_i > 2$ for any node $v_i$, we have
\begin{equation}
    \begin{aligned}
        \mathbb{E}_{v_i}[\|\mathbf{b}_i^{(l+1)} - \mathbf{\bar \mu}\|_2^2]
        & = \mathbb{E}_{v_i \in \mathcal{V}} \left[ \left(\hat{\mathbf{b}}_i^{\left(l+1\right)}\right)^T\hat{\mathbf{b}}_i^{\left(l+1\right)} \right] \\
        & \leq \mathbb{E}_{v_i \in \mathcal{V}} \left[ \frac{2}{d_i}\right] \left( r_0 \mathbb{E}_{v_j \in \mathcal{V}_0}\left[ \left(\hat{\mathbf{b}}_{j, 0}^{\left(l\right)}\right)^T \hat{\mathbf{b}}_{j, 0}^{\left(l\right)}\right] + r_1 \mathbb{E}_{v_j \in \mathcal{V}_1}\left[ \left(\hat{\mathbf{b}}_{j, 1}^{\left(l\right)}\right)^T \hat{\mathbf{b}}_{j, 1}^{\left(l\right)}\right]\right) \\
        & < \left( r_0 \mathbb{E}_{v_j \in \mathcal{V}_0}\left[ \left(\hat{\mathbf{b}}_{j, 0}^{\left(l\right)}\right)^T \hat{\mathbf{b}}_{j, 0}^{\left(l\right)}\right] + r_1 \mathbb{E}_{v_j \in \mathcal{V}_1}\left[ \left(\hat{\mathbf{b}}_{j, 1}^{\left(l\right)}\right)^T \hat{\mathbf{b}}_{j, 1}^{\left(l\right)}\right]\right) \\
        & < \mathbb{E}_{v_i \in \mathcal{V}} \left[ \left(\hat{\mathbf{b}}_i^{\left(l\right)}\right)^T\hat{\mathbf{b}}_i^{\left(l\right)} \right] \\
        & < \mathbb{E}_{v_i}\left[\|\mathbf{b}_i^{(l)} - \mathbf{\bar \mu}\|_2^2\right]
    \end{aligned}
\end{equation}
which completes the proof.

\section{Pseudocode of \method} \label{sec:alg_appendix}
The pseudocode of \method~is presented in Algorithm~\ref{alg:bemap}. Before training, we precompute the sampling probability in the balance-aware sampling (lines 3 -- 5). During each epoch, we first generate the fair neighborhood using the balance-aware sampling (lines 7 -- 11). Then for each hidden layer, the fair node representation of each node is learned on the fair neighborhood (lines 12 -- 15). Finally, we update the model parameters with back-propagation (lines 16 -- 17).
\begin{algorithm}[h!]
	\SetKwInOut{Input}{Input}
	\SetKwInOut{Output}{Output}
	\Input{An input graph $\mathcal{G} = \left\{\mathcal{V}, \mathbf{A}, \mathbf{X}\right\}$, a set of training nodes $\mathcal{V}_{\text{train}}$, ground-truth labels $\mathcal{Y}_{\text{train}}$, an $L$-layer GCN with weight matrices $\mathbf{\Theta}=\left\{\mathbf{W}^{\left(1\right)},\ldots,\mathbf{W}^{\left(L\right)}\right\}$, a task-specific loss function $J$, maximum number of epochs ${\rm epoch}_{\rm max}$, hyperaprameters $\beta$, $m$, $\delta$;}
	\Output{A well trained GCN $\mathbf{\Theta}=\left\{\mathbf{W}^{(1)},\ldots,\mathbf{W}^{\left(L\right)}\right\}$.}
	Initialize $\mathbf{H}^{\left(1\right)} = \mathbf{X}$\;
	Initialize the gradient-based optimizer $\textit{OPT}$\;
	\For{each node $v_i\in\mathcal{V}$}{
	    \For{each node $v_j\in\mathcal{N}\left(v_i\right)$}{
	        Precompute $P(v_j|v_i)$ by Eqs.~\eqref{eq:balance_score} and~\eqref{eq:balance_aware_sampling}\;
	    }
	}
	\For{${\rm epoch}=1\rightarrow {\rm epoch}_{\rm max}$}{
	    \tcp{Fair neighborhood generation}
	    \For{each node $v_i\in\mathcal{V}$}{
	        \eIf{$\forall v_j\in\mathcal{\widehat N}\left(v_i\right)$ belongs to the same demographic group}{
	            Sample the fair neighborhood $\mathcal{\widehat N}^{\text{fair}}\left(v_i\right) = \mathcal{N}^{\text{fair}}\left(v_i\right)\cup \left\{v_i\right\}$ with probability $P(v_j|v_i),\forall v_j\in\mathcal{N}\left(v_i\right)$ such that $|\mathcal{N}^{\text{fair}}\left(v_i\right)| = \max\left\{\beta|\mathcal{N}\left(v_i\right)|, m\right\}$\;
	        }{
	            Sample the neighborhood $\mathcal{N}^{\text{fair}}\left(v_i\right)$ with probability $P(v_j|v_i),\forall v_j\in\mathcal{N}\left(v_i\right)$ and generate the fair neighborhood $\mathcal{\widehat N}^{\text{fair}}\left(v_i\right) = \mathcal{N}^{\text{fair}}\left(v_i\right)\cup \left\{v_i\right\}$ such that $|\mathcal{\widehat N}_0^{\text{fair}}\left(v_i\right)| = |\mathcal{\widehat N}_1^{\text{fair}}\left(v_i\right)|$
	        }
	    }
	    \tcp{Forward propagation}
	    \For{each hidden layer $l\in\left\{1,\ldots,L\right\}$}{
	        \For{each node $v_i\in\mathcal{V}$}{
	            $\mathbf{\widehat h}^{(l)}_i = \sum_{v_j \in \mathcal{\widehat N}\left(v_i\right)} \alpha_{ij} \mathbf{h}_j^{(l)}$ with $\alpha_{ij} = \frac{1}{|\mathcal{\widehat N}^{\text{fair}}\left(v_i\right)|}$ if row normalization else $\alpha_{ij} = \frac{1}{\sqrt{|\mathcal{\widehat N}^{\text{fair}}\left(v_i\right)|}\sqrt{|\mathcal{\widehat N}^{\text{fair}}(v_j)|}}$\;
	            $\mathbf{h}_i^{(l+1)} = \sigma\big(\mathbf{\widehat h}^{(l)}_i \mathbf{W}^{(l)}\big)$;
	        }
	    }
	    \tcp{Backward propagation}
	    Calculate the empirical loss $\text{loss} = J(\mathcal{V}_{\text{train}}, \mathcal{Y}_{\text{train}}, \left\{\mathbf{h}_i^{(L+1)},\forall v_i\in\mathcal{V}\right\})$\;
        Update $\mathbf{\Theta}$ by $\textit{OPT}(\nabla \text{loss})$\;
	}
	\Return $\mathbf{\Theta}=\left\{\mathbf{W}^{(1)},\ldots,\mathbf{W}^{(L)}\right\}$\;
	\caption{Training GCN with \method.}
	\label{alg:bemap}
\end{algorithm}

\section{Extension of \method to Non-binary Sensitive Attribute} \label{sec:nonbinary_appendix}
We consider a non-binary sensitive attribute $s$ which forms $n_s$ demographic groups, i.e., $s\in\left\{1,\ldots, n_s\right\}$. The key idea of \method~is to balance the number of neighbors across different demographics. We discuss two cases for balancing the neighborhood in the following.
\begin{itemize}[
    align=left,
    leftmargin=2em,
    itemindent=0pt,
    labelsep=0pt,
    labelwidth=1em,
]
    \item \textit{When all neighbors belong to one demographic group:} We adopt the same strategy as the case of binary sensitive attribute by sampling a subset of $k$ neighbors for any node $v_i$ such that $k = \max\left\{\beta|\mathcal{N}\left(v_i\right)|, 4\right\}$. 
    \item \textit{When the neighbors belong to different demographic groups:} In this case, for any node $v_i$, we first count the number of neighbors in each demographic group $\left\{\mathcal{\widehat N}_s\left(v_i\right) \vert \forall s = 1, \ldots, n_s\right\}$. Then we set the number of neighbors to be sampled for any demographic group $k$ as the smallest non-zero value in $\left\{\mathcal{\widehat N}_s\left(v_i\right) \vert \forall s = 1, \ldots, n_s\right\}$. After that, for each demographic group in the neighborhood of $v_i$, we sample $k$ neighbors to create a balanced neighborhood.
\end{itemize}
Regarding the balance-aware sampling probability in the above sampling steps, we modify Eq.~\eqref{eq:balance_score} by replacing the absolute difference $|{\widetilde N}_0\left(v_i\right) - {\widetilde N}_1\left(v_i\right)|$ in the cardinalities of two demographic groups in binary case to the average squared difference between the cardinalities of any two demographic groups.
\begin{equation}\label{eq:balance_score_nonbinary}
\begin{aligned}
& \text{balance}_i = \frac{1}{\sqrt{\frac{2}{n_s\left(n_s-1\right)} \sum_{s=1}^{n_s} \left(\mathcal{\widehat N}_k\left(v_i\right) - \mathcal{\widehat N}_j\left(v_i\right)\right)^2} + \delta} \\
& \forall k,j\in\left\{1, \ldots, n_s\right\}, k \neq j
\end{aligned}
\end{equation}
It should be noted that when $s$ is binary sensitive attribute, Eq.~\eqref{eq:balance_score_nonbinary} is equivalent to Eq.~\eqref{eq:balance_score}. In this way, we could adopt similar training procedure in Algorithm~\ref{alg:bemap} to train the fair graph neural network.

\section{Detailed Experimental Settings}
\subsection{Dataset Descriptions} \label{sec:appendix_datasets}
Here, we provide detailed descriptions for \textit{Pokec-z}~\cite{nguyen2021graphdta}, \textit{NBA}~\cite{dai2021say}, \textit{Credit}~\cite{dong2022edits}, and \textit{Recidivism}~\cite{agarwal2021towards}. 
The \textit{Pokec-z} dataset \cite{nguyen2021graphdta} is drawn from Pokec, a Facebook-like social network in Slovakia, based on regional information. In Pokec-z, we set the sensitive attribute as the region of a user and aim to predict the field of work of a user.
The \textit{NBA} dataset~\cite{dai2021say} includes the age, nationality (US vs. overseas), salary of a NBA player for the 2016 -- 2017 season. Two players are connected if one follows another on Twitter. In this dataset, the nationality is used as the sensitive attribute, and the goal is to predict whether the salary of a player is above the median.
The \textit{Credit} dataset~\cite{dong2022edits} contains age, education and payment patterns of a credit card customer. The links among customers are determined by the pairwise similarity of their credit accounts. Here, age is set as the sensitive attribute and the label is whether a user will default on credit card payments. 
The \textit{Recidivism} dataset~\cite{agarwal2021towards} consists of defendants who were released on bail during 1990 -- 2009. Two defendants are connected if they share similar demographic information and criminal records. The goal is to predict whether a defendant is likely to commit a crime when released (i.e., bail) or not (i.e., no bail) with race as the sensitive attribute.
The detailed statistics of the four datasets are provided in Table~\ref{tbl:statics}.

\begin{table}[h]
    \centering
    \caption{The statistics of datasets.}
    \begin{tabular}{c|cccc}
        \hline
         & Pokec-z & NBA & Recidivism & Credit  \\
        \hline 
        \# Nodes & 5,631 & 313 & 9,538 & 21,000 \\
        \# Edges & 17,855 & 14,543 & 167,930 & 160,198 \\
        \# Attributes & 59 & 39 & 18 & 13 \\
        \# Classes & 2 & 2 & 2 & 2 \\
        Avg. Degrees & 6.4 & 92.9 & 35.2 & 15.3 \\
        Sensitive Attr. & Region & Country & Race & Age \\ 
        Label & Working field & Salary & Bail decision & Future default \\
        \hline
    \end{tabular}
    \label{tbl:statics}
\end{table}

\subsection{Descriptions of Baseline Methods}\label{sec:appendix_baseline_methods}
Regarding graph neural networks without fairness considerations, 
\begin{itemize}[
    align=left,
    leftmargin=2em,
    itemindent=0pt,
    labelsep=0pt,
    labelwidth=1em,
]
    \item \textit{GCN}~\cite{kipf2016semi} learns the representation of a node by iteratively aggregating the representations of its 1-hop neighbors; 
    \item \textit{GraphSAGE}~\cite{hamilton2017inductive} aggregates node representations from a subset of 1-hop neighbors.
\end{itemize}

For fair graph neural networks,
\begin{itemize}[
    align=left,
    leftmargin=2em,
    itemindent=0pt,
    labelsep=0pt,
    labelwidth=1em,
]
    \item \textit{FairGNN}~\cite{dai2021say} leverages adversarial learning to debias the node representations;
    \item \textit{EDITS}~\cite{dong2022edits} removes bias in the input data by minimizing the Wasserstein distance;
    \item \textit{FairDrop}~\cite{spinelli2021biased} mitigates bias by randomly masking edges in the input graph; 
    \item \textit{NIFTY}~\cite{agarwal2021towards} debias the graph neural networks with a contrastive learning framework; 
    \item \textit{FMP}~\cite{jiang2022fmp} redesigns the message passing schema in Graph Convolutional Network for bias mitigation.
\end{itemize}

\section{Additional Experimental Results}
\begin{table}[htbp]
    \centering
    \caption{
    Relative reduction with respect to $\spmetric$ and $\eometric$ for \method~on Pokec-z, NBA, Recidivism, and Credit. For \method~(row) and \method~(sym), the relative reduction is computed by comparing to the vanilla GCN. For \method~(gat), the relative reduction is computed by comparing to the vanilla GAT.}
    \resizebox{\linewidth}{!}{
    \begin{tabular}{c|c|c|c|c}
    \hline
        \multirow{2}{*}{Methods} & \multicolumn{2}{c}{Pokec-z} & \multicolumn{2}{c}{NBA} \\
        \cline{2-5}
        & $\text{Reduction}_{\text{SP}}$ (\%) $\uparrow$ &$\text{Reduction}_{\text{EO}}$ (\%) $\uparrow$ & $\text{Reduction}_{\text{SP}}$ (\%) $\uparrow$ & $\text{Reduction}_{\text{EO}}$ (\%) $\uparrow$ \\
        \hline
        \method~(row) & 91.7 & 80.0 & 2.25 & 68.8 \\
        \method~(sym) & 83.6 & 86.8 & 0.12 & 70.8 \\
        \method~(gat) & 91.6 & 89.5 & 32.8 & 57.9 \\
    \hline
    \end{tabular}
    }
    \resizebox{\linewidth}{!}{
    \begin{tabular}{c|c|c|c|c}
    \hline
        \multirow{2}{*}{Methods} & \multicolumn{2}{c}{Recidivism}  & \multicolumn{2}{c}{Credit} \\
        \cline{2-5}
        & $\text{Reduction}_{\text{SP}}$ (\%) $\uparrow$ & $\text{Reduction}_{\text{EO}}$ (\%) $\uparrow$ & $\text{Reduction}_{\text{SP}}$ (\%) $\uparrow$ & $\text{Reduction}_{\text{EO}}$ (\%) $\uparrow$\\
        \hline
        \method~(row) & 66.3 & 63.7 & 41.4 & 29.5 \\
        \method~(sym) & 79.3 & 53.4 & 40.1 & 45.7 \\
        \method~(gat) & 50.4 & 27.8 & 27.2 & 29.6 \\
    \hline
    \end{tabular}
    }
    \label{tab:percentage_of_deline_appendix}
\end{table}

\noindent \textbf{Relative Reductions with respect to $\spmetric$ and $\eometric$.} We present the relative reduction of $\spmetric$ and $\eometric$ for \method~in Table~\ref{tab:percentage_of_deline_appendix}. The relative reductions with respect to $\spmetric$ and $\eometric$ are defined as follows.
\begin{equation}
\begin{aligned}
    \text{Reduction}_{\text{SP}} = \left(1 - \frac{\spmetric^{\text{\method}}}{\spmetric^{\text{vanilla}}}\right) \times 100\% \\
    \text{Reduction}_{\text{EO}} = \left(1 - \frac{\eometric^{\text{\method}}}{\eometric^{\text{vanilla}}}\right) \times 100\%
\end{aligned}
\end{equation}
where $\spmetric^{\text{\method}}$ and $\eometric^{\text{\method}}$ are $\spmetric$ and $\eometric$ for the proposed \method, respectively, and $\spmetric^{\text{vanilla}}$ and $\eometric^{\text{vanilla}}$ are $\spmetric$ and $\eometric$ for the corresponding vanilla graph neural network, i.e., vanilla GCN for \method~(row) and \method~(sym) and vanilla GAT for \method~(gat), respectively.
From Table~\ref{tab:percentage_of_deline_appendix}, we observe that \method~consistently reduce more than 25\% of $\spmetric$ and $\eometric$ on all datasets, except for \method~(row) and \method~(sym) on NBA. On Pokec-z, \method~could even reduce more than 80\% of $\spmetric$ and $\eometric$ compared to vanilla graph neural networks without fairness consideration.

\begin{table}[htbp]
    \centering
    \caption{Semi-supervised node classification on the Pokec-n. Higher is better ($\uparrow$) for ACC and AUC (white). Lower is better ($\downarrow$) for $\spmetric$ and $\eometric$ (gray). Bold font indicates the best performance for fair graph neural networks, and underlined number indicates the second best.}
    \begin{tabular}{c|c|c|c|c}
        \hline
         \multirow{2}{*}{\textbf{Methods}} & \multicolumn{4}{c}{Pokec-n}\\
         \cline{2-5}
          & ACC (\%) $\uparrow$ & AUC (\%) $\uparrow$ & \cellcolor{black!15} $\spmetric$ (\%) $\downarrow$ & \cellcolor{black!15} $\eometric$ (\%) $\downarrow$ \\
          \hline
          GCN & 68.41 $\pm$ 0.30 & 73.4 $\pm$ 0.09 & \cellcolor{black!15} 6.83 $\pm$ 1.09 & \cellcolor{black!15} 9.59 $\pm$ 1.16 \\
          GraphSAGE & 65.41 $\pm$ 0.74 & 69.35 $\pm$ 0.76 & \cellcolor{black!15} 6.81 $\pm$ 0.47 & \cellcolor{black!15} 11.20 $\pm$ 0.84 \\
          GAT & 65.81 $\pm$ 0.43 & 69.30 $\pm$ 0.86 & \cellcolor{black!15} 6.63 $\pm$ 0.68 & \cellcolor{black!15} 10.34 $\pm$ 1.37 \\ 
          \hline
          FairGNN & 68.20 $\pm$ 0.87 & \underline{72.92 $\pm$ 0.18} &\cellcolor{black!15} 8.73 $\pm$ 1.37 &\cellcolor{black!15} 11.02 $\pm$ 1.32 \\
          EDITS & 66.82 $\pm$ 0.98 & 70.39 $\pm$ 0.02 &\cellcolor{black!15} \textbf{4.32 $\pm$ 0.43} &\cellcolor{black!15} 6.20 $\pm$ 0.74 \\
          NIFTY & 66.91 $\pm$ 1.38 & 71.27 $\pm$ 0.32 & \cellcolor{black!15}7.72 $\pm$ 1.06 &\cellcolor{black!15} \textbf{5.83 $\pm$ 0.43}\\
          FMP & 67.62$\pm$1.27 & \textbf{77.40$\pm$0.23} &\cellcolor{black!15} 32.78$\pm$1.89 & \cellcolor{black!15}29.67$\pm$2.51 \\
          \hline
          \method~(row) & \underline{68.39 $\pm$ 0.52} & 71.58 $\pm$ 0.61 &\cellcolor{black!15} 5.63 $\pm$ 2.65 &\cellcolor{black!15} \underline{6.11 $\pm$ 1.85} \\
          \method~(sym) & \textbf{69.12 $\pm$ 0.76} & 71.59 $\pm$ 1.12 &\cellcolor{black!15} \underline{5.45 $\pm$ 0.62} &\cellcolor{black!15} 7.97 $\pm$ 0.61\\
          \method~(gat) & 65.47 $\pm$ 0.97 & 67.95 $\pm$ 0.39 &\cellcolor{black!15} 5.67 $\pm$ 0.56 &\cellcolor{black!15} 8.73 $\pm$ 0.50 \\
        \hline
    \end{tabular}
    \label{tab:pokec_n_results}
\end{table}

\noindent \textbf{Additional Results on Pokec-n.} We conduct experiments on another social network named Pokec-n. Similar to Pokec-z, it is also drawn from the Slovakian social network Pokec, but focuses on different province in the country compared to Pokec-z. Additionally, the sensitive attribute is selected as the region of a user and aim to predict the field of work of a user, which is consistent with the settings in Pokec-z. Experimental results on Pokec-n is shown in Table~\ref{tab:pokec_n_results}. From the table, \method~can still consistently mitigate bias, despite being the second bias method in terms of bias mitigation (as shown by $\spmetric$ and $\eometric$). In the meanwhile, \method~achieves the best classification results (as shown by ACC and AUC) while mitigating bias, which demonstrate that \method~achieves the best trade-off between fairness and utility.

\section{More Related Works} \label{appdix:related_work}
Here, we discuss more related works in designing fair message passing, and a few related work on class-imbalanced graph learning, graph rewiring and graph sampling. 

\noindent \textbf{Fair message passing} ensures fairness on graphs by either preprocessing the input graph or redesigning message passing in graph neural networks~\cite{zhu2023fairness, li2021dyadic, subramonian2022discrimination}. \cite{zhu2023fairness} views message passing as the solution to an optimization and introduces Maximum Mean Discrepancy (MMD) as a regularization in the optimization problem to redesign the fairness-aware message passing. \cite{li2021dyadic} learns a fair adjacency matrix for link prediction by aligning the edge weights between intra-group edges and inter-group edges. \cite{subramonian2022discrimination} reduces the discrimination risk, in order to ensure fairness on mean aggregation feature imputation. Our work bears subtle differences existing works. Compared to \cite{zhu2023fairness}, we do not change how neighborhood aggregation procedure in message passing, but change the neighborhood itself to ensure fairness. Compared to \cite{li2021dyadic}, we focus on fairness in node classification rather than dyadic fairness in link prediction. Compared to \cite{subramonian2022discrimination}, we only perform sampling on the adjacency matrix and do not consider changing node features.

\noindent \textbf{Class-imbalanced graph learning} refers to graph learning with uneven label distribution, i.e., one class has a significantly higher number of data than other classes~\cite{zhao2021graphsmote, park2022graphens, qu2021imgagn, chen2021topology}. For example, \cite{park2022graphens} mixes up the nodes from the minority class and selected target nodes. \cite{qu2021imgagn} generates synthetic nodes from the minority class via generative adversarial networks (GAN)~\cite{goodfellow2020generative}. \cite{chen2021topology} reweighs the influence of labeled node adaptively based on their distances to class boundaries. \cite{zhao2021graphsmote} generalizes SMOTE~\cite{chawla2002smote} to graphs and synthesizes new samples in the embedding space. It should be noted that, different from this line of work that focuses on the imbalance with respect to class label, our work focuses on the imbalance with respect to a sensitive attribute. For more related work, please refer to \cite{ma2023class}.

\noindent \textbf{Graph rewiring} changes the graph topology by rewiring edges in the graph (i.e., delete an edge from a source node to a target node and add an edge between the source node to a different target node) to find a better graph topology for the learning task~\cite{ma2019attacking, bi2022make, arnaiz2022diffwire, chan2016optimizing}. \cite{chan2016optimizing} proposes degree-preserving edge rewiring to maximally improve its robustness under a certain budget. Different from \cite{chan2016optimizing}, our method focuses on the balance of neighborhoods with respect to a sensitive attribute for fairness rather than node degrees. \cite{ma2019attacking} rewires the input graph via reinforcement learning for adversarial attacks on graphs, whereas our work focuses on learning fair graph neural network. \cite{bi2022make} reduces heterophily by rewiring graphs in consideration of pairwise similarity of label/feature-distribution between a node and its neighbors. Compared to \cite{bi2022make}, we focuses on generating balanced neighborhood with respect to a sensitive attribute instead of reducing heterophily with respect to class labels. \cite{arnaiz2022diffwire} incorporates the Lovász bound into graph rewiring to overcome over-smoothing, over-squashing and under-reaching in graph neural networks, while our work focuses on fairness in node classification.

\noindent \textbf{Graph sampling} aims to reduce computational complexity by sampling a subset of edges in the graph. 
\cite{zeng2019graphsaint} preserves neighbors with higher influence for node representation learning with better utility, wheres our method applies sampling to generate balanced neighborhoods for fair representation learning. 
\cite{7752223} considers node degrees and selects edges based on the ranks of node degrees in an deterministic way. Different from it, our work does not consider node degree but the sensitive attribute of the neighbors and samples the neighborhood in a non-deterministic way.
\cite{NEURIPS2018_01eee509} develops an adaptive layer-wise sampling method by sharing the same neighborhood in low layers with different parent nodes in high layers. Different from it, we do not associate the same neighborhood with different nodes across layers, and the balanced neighborhood keeps the same across all layers.

\end{document}